\documentclass[twoside]{article}

\usepackage{amsmath,amsthm,amssymb}
\usepackage{olo}
\usepackage{hyperref}
\usepackage{subcaption}
\usepackage{grffile}
\usepackage{fullpage}

\newtheorem{definition}{Definition}
\newtheorem{theorem}{Theorem}
\newtheorem{lemma}[theorem]{Lemma}
\newtheorem{claim}[theorem]{Claim}

\graphicspath{{plots/}}

\newcommand{\vwo}{\vw^\ast}
\newcommand{\bto}{\vwo}
\newcommand{\btt}{\hat\vw^t}
\newcommand{\vwn}{\vw^+}
\newcommand{\xtt}{\hat\vx^t}
\newcommand{\xto}{\vx^{t,\ast}}
\newcommand{\tvw}{\tilde\vw}
\newcommand{\hvw}{\hat\vw}
\newcommand{\ols}{\textsf{OLS}\xspace}
\newcommand{\irls}{\textsf{IRLS}\xspace}
\newcommand{\stir}{\textsf{STIR}\xspace}
\newcommand{\girls}{\textsf{STIR-GD}\xspace}
\newcommand{\torrent}{\textsf{TORRENT}\xspace}
\newcommand{\tgd}{\textsf{TORRENT-GD}\xspace}
\newcommand{\wucb}{\textsf{WUCB-Lin}\xspace}
\newcommand{\rucb}{\textsf{RUCB-Lin}\xspace}
\newcommand{\<}{\leftarrow}

\makeatletter
\newcommand{\newreptheorem}[2]{\newtheorem*{rep@#1}{\rep@title}
\newenvironment{rep#1}[1]{\def\rep@title{#2 \ref*{##1}}\begin{rep@#1}}{\end{rep@#1}}
}
\makeatother

\newcounter{tmpcounter}
\newcounter{modelcounter}

\newreptheorem{lemma}{Lemma}
\newreptheorem{theorem}{Theorem}
\newreptheorem{claim}{Claim}

\title{Globally-convergent Iteratively Reweighted Least Squares\\for Robust Regression Problems}
\author{Bhaskar Mukhoty$^\dagger$ \and Govind Gopakumar$^\dagger$\thanks{Work done as a master's student at IIT Kanpur} \and Prateek Jain$^\ddagger$ \and Purushottam Kar$^\dagger$\\$^\dagger$IIT Kanpur\\$^\ddagger$Microsoft Research India\\\texttt{\{bhaskarm,govindg,purushot\}@cse.iitk.ac.in}, \texttt{prajain@microsoft.com}}

\begin{document}

\maketitle

\begin{abstract}
We provide the first global model recovery results for the \irls (iteratively reweighted least squares) heuristic for robust regression problems. \irls is known to offer excellent performance, despite bad initializations and data corruption, for several parameter estimation problems. Existing analyses of \irls frequently require careful initialization, thus offering only local convergence guarantees. We remedy this by proposing augmentations to the basic \irls routine that not only offer guaranteed global recovery, but in practice also outperform state-of-the-art algorithms for robust regression. Our routines are more immune to hyperparameter misspecification in basic regression tasks, as well as applied tasks such as linear-armed bandit problems. Our theoretical analyses rely on a novel extension of the notions of strong convexity and smoothness to \emph{weighted strong convexity and smoothness}, and establishing that sub-Gaussian designs offer bounded \emph{weighted condition numbers}. These notions may be useful in analyzing other algorithms as well.
\end{abstract}

\section{Introduction}
\label{sec:intro}
Suppose there exists an unknown gold model $\vwo$ and we are given $n$ data points $(\vx_i,y_i)_{i=1}^n$ with $d$-dimensional covariates $\vx_i \in \bR^d$ and the real-valued responses $y_i$ generated as $y_i = \vx_i^\top\vwo$. However, for an unknown set of $k < n$ data points $i_1, \ldots i_k$, the responses get corrupted i.e. we instead receive $y_{i_j} = \vx_{i_j}^\top\vwo + b_{i_j}$ where $b_{i_j} \in \bR$ is the corruption. Given the complete set of clean and corrupted data points $(\vx_i,y_i)_{i=1}^n$, can we recover the gold model $\vwo$?

This is the classical robust regression problem that has become increasingly relevant to machine learning and statistical estimation techniques which frequently encounter situations where data is not trustworthy. Works exist in settings where test data is corrupted in order to fool a model that was learnt on clean data \cite{GoodfellowMP2018}, as well as the more challenging setting, on which we focus, where the training data presented to the algorithm is itself corrupted \cite{CandesLW2009, ChenCM2013, FengXMY2014}.

We will seek to offer reliable model recovery despite the presence of (possibly maliciously) corrupted data in the training set. Settings which present corrupted data to learning algorithms include relatively innocuous instances of erasures and missing data, improperly or mistakenly attributed data, transient or temporary changes in user-behavior patterns, as well as deliberate and malicious attempts to derail recommendation systems and other decision-making systems using malware, click-bots and other fraudulent techniques.

Despite being a well established field, given the early seminal contributions of Huber \cite{Huber1964} and Tukey \cite{Tukey1960}, robust statistics and algorithms have received renewed interest given the threat to modern machine learning techniques. Of the several techniques that have been proposed for robust learning problems, one heuristic, namely the iteratively reweighted least squares (\irls), remains a practitioner's favorite owing to its ease of use and excellent performance. The \irls technique has been effectively adapted to several problems, including sparse recovery, and robust regression. The work of \cite{StaszakV2016} shows that certain biological dynamical systems can be modeled upon the IRLS principle as well.

\subsection{Our Contributions}
We offer several advances in the understanding and application of the \irls method. In particular, we provide the first global model recovery guarantee for \irls for robust regression - our contributions are distinguished in the context of existing analyses for \irls in \S\ref{sec:related}. We also propose algorithmic augmentations, in particular a fast gradient-based variant, to the basic \irls heuristic which offer superior performance compared to existing state-of-the-art robust algorithms in terms of speed, as well as resilience to misspecified hyperparameters. We demonstrate this in the standard linear regression setting, as well as an applied setting, namely linear-armed bandits.

\begin{table*}[t]%
\centering
\caption{\small Algorithms for the Robust Regression problem (corrupted responses). $^\text{\textdagger}$Please see \S\ref{sec:formulation} for details. Algorithms able to tolerate adaptive (as opposed to oblivious) adversaries are more resilient. A more robust algorithm can handle larger $\alpha$. Sub-Gaussian covariates offer a much more flexible model than (isotropic) Gaussian covariates.}
\resizebox{\columnwidth}!{\begin{tabular}{ccccc}
\hline
Paper & Adversary Model$^\text{\textdagger}$ & Breakdown point$^\text{\textdagger}$ & Covariate Model & Technique\\\hline
Bhatia et. al. 2015 \cite{BhatiaJK2015} & Adaptive & $\alpha \geq \Om1$ & sub-Gaussian & Hard Thresholding (fast)\\\hline
Chen \& Dalalyan 2010 \cite{ChenD2012} & Adaptive & $\alpha \geq \Om1$ & sub-Gaussian & SOCP (slow)\\\hline
Wright \& Ma 2010 \cite{WrightM2010} & Oblivious & $\alpha \rightarrow 1$ & Isotropic Gaussian & $L_1$ regularization (slow)\\\hline
\textbf{This Paper} & \textbf{Adaptive} & $\mathbf{\valpha \geq \Om1}$  & \textbf{sub-Gaussian} & \textbf{Reweighting (fast)}\\\hline
\end{tabular}
}
\label{tab:rreg}
\end{table*}

\section{Related Work}
\label{sec:related}
Two lines of work directly relate to our contributions: 1) robust algorithms for regression and other learning problems, and 2) works that analyze (variants of) the \irls heuristic in various settings. We review both, as well as distinguish our contributions, below.

\noindent\textbf{Robust Learning Algorithms}: Work on robust statistics dates back several decades \cite{Huber1964,Tukey1960} and is too vast to be reviewed in detail. Recent years have seen interest in scalable algorithms for classification \cite{FengXMY2014}, principal component analysis \cite{CandesLW2009}, and moment estimation \cite{DiakonikolasKKLMS2017b}. Within the specific problem of robust regression, two broad lines of work exist:

\noindent\textit{Covariate (feature) corruption}: Results in this setting usually either give only weak guarantees, or else severely constrain data. e.g., \cite{ChenCM2013,McWilliamsKLB14} allow only a $\bigO{1/\sqrt d}$ fraction of data to be corrupted, $d$ being the ambient dimensionality, whereas \cite{DiakonikolasKS2019,LiuSLC2018} only admit covariates drawn from a Gaussian distribution.

\noindent\textit{Response (label) corruption}: Variants within this setting arise based on the power of the adversary introducing the corruptions, the fraction of data points that can be corrupted, restrictions on the choice of covariates, and scalability of the algorithms. Table~\ref{tab:rreg} summarizes these traits for a selection of algorithms. We refer the reader to \cite{BhatiaJK2015,DiakonikolasKS2019} for other references.

\noindent\textbf{\irls Variants and Analyses}: The \irls heuristic has been successfully applied to several problems including sparse recovery \cite{BaBPB2013, DaubechiesDFG2010}, facility location problems \cite{BrimbergLove1993} (via the Weiszfeld procedure), and optimizing various robust cost functions, such as the $L_q$ and Huber loss functions \cite{AftabHartley2015,BissantzDMS2009,Cline1972,Osborne1985}. 

Some of these works are not directly relevant to robust regression as they either operate with uncorrupted data \cite{BrimbergLove1993}, or else assume that the noise is Gaussian \cite{BaBPB2013, DaubechiesDFG2010}. Convergence guarantees for \irls are common in these benign settings.
To handle adversarial corruptions, it is common to use \irls to optimize a \emph{robust cost function} $F$ such as $L_q$ or Huber loss, in the anticipation that the model so obtained, say $\hvw = \arg\min F(\vw; \bc{(\vx_i,y_i)})$, will ensure $\hvw \approx \vwo$.

However, none of these works actually ensure such a result i.e. $\hvw \approx \vwo$. Some works \cite{BissantzDMS2009, Cline1972, Osborne1985} operate with cost functions that are convex (e.g. $L_q$ for $q \in [1,2]$) and simply show that \irls approaches small cost function values. Other approaches \cite{AftabHartley2015} do work with non-convex cost functions, but then offer only monotonicity guarantees and no global convergence guarantees.

We bridge this gap by presenting a much stronger analysis of \irls that guarantees \emph{global recovery} of the gold model $\vwo$ under mild conditions. Key to our proof technique is a novel concept that extends the basic notions of strong convexity and strong smoothness to \emph{weighted} versions of the same, as well as a guarantee that Gaussian and sub-Gaussian designs have bounded \emph{weighted condition numbers}. These results may be of independent interest in analyzing other algorithms.

\section{Notation}
\label{sec:notation}
Bold lower-case Latin letters $\vx, \vy$ denote vectors. $\vx_i$ denotes the $i\nth$ coordinate of the vector $\vx$. Upper case Latin letters $A, X$ denote matrices. For a vector $\vv \in \bR^n$ and set $S \subset [n]$, $\vv_S$ denotes the vector with $(\vv_S)_i = \vv_i$ for $i \in S$ and $(\vv_S)_j = 0$ for $j \notin S$. Similarly, for any matrix $A \in \bR^{d \times n}$ and any set $S \subset [n]$, $A_S$ denotes the matrix in which columns $i \in S$ in $A_S$ are identical to those in $A$ and columns $j \notin S$ are filled with zeros.

$\lambda_{\min}(X)$ and $\lambda_{\max}(X)$ denote, respectively, the smallest and largest eigenvalues of a square symmetric matrix $X$. $\cB_2(\vv,r) := \bc{\vx \in \bR^d: \norm{\vx-\vv}_2 \leq r}$ denotes the ball of radius $r$ centered at $\vv$. $S^{d-1}$ denotes the surface of the unit sphere in $d$ dimensions. We use the shorthand $\cB_2(r) := \cB_2(\vzero,r)$.
\section{Problem Formulation}
\label{sec:formulation}
Given $n$ data points $(\vx_i,y_i) \in \bR^d \times \bR$, let $R_X := \max_{i \in [n]}\ \norm{\vx_i}_2$ be the maximum Euclidean length of any covariate, $X = [\vx_1,\ldots,\vx_n] \in \bR^{d\times n}$ be the covariate matrix, and $\vy = [y_1,\ldots,y_n]^\top \in \bR^n$ the response vector. Assume that the covariates are generated as $\vx_1,\ldots,\vx_n \sim \cD$ from an unknown distribution $\cD$ with mean $\vmu \in \bR^d$ and sub-Gaussian norm \cite{Vershynin2018} $\norm{\cD}_{\Psi_2} \leq R$. $\vwo \in \bR^d$ will be the gold model with $R_W := \norm{\vwo}_2$.

\noindent\textbf{Noise Model}: Given the data covariates and the gold model, the responses are generated as $\vy = X^\top\vwo + \vb$ where $\vb = [b_1,\ldots,b_n]$ is the vector of corruptions. We make the standard assumption that $\norm\vb_0 \leq \alpha\cdot n$. Let $B := \supp(\vb)$ denote the ``bad'' points which suffer corruption i.e. $\vb_j \neq 0$ for $j \in B$ (note that $\abs{B} \leq \alpha\cdot n$) and $G = [n] \setminus B$ denote the ``good'' points where $\vb_i = 0$ and thus $y_i = \vx_i^\top\vwo$ for $i \in G$. To avoid clutter, we abuse notation to denote $G := \abs{G}$ and $B := \abs{B}$. The largest value of the corruption fraction $\alpha$ that an algorithm can tolerate is known as its \emph{breakdown point}.  

\noindent\textbf{Adversary Model}: We will work with a partially adaptive adversary which is compelled to choose locations of the corruptions $\supp(\vb) = B$ before any data covariates have been generated or $\vwo$ is revealed. However, the adversary may fill in the corruption values at those locations with knowledge of $\vwo$ and $X$. Our results can be extended to a \emph{fully adaptive adversary} that choose $\supp(\vb)$ after looking at $\vwo$ and $X$ as well, but at a cost of a smaller breakdown point $\alpha$.

Key to our analyses are the notions of \emph{weighted strong convexity and smoothness} which we define below. These definitions reflect the fact that \irls solves \emph{weighted} regression problems iteratively.

\begin{definition}[WSC/WSS]
We say that a covariate matrix $X \in \bR^{d \times n}$ offers \emph{weighted strong convexity} (WSC) at level $\lambda_S$ (resp. \emph{weighted strong smoothness} (WSS) at level $\Lambda_S$), with respect to a diagonal weight matrix $S = \diag(\vs) \in \bR^{n \times n}$ where $\vs_i \geq 0, i \in [n]$, if
\[
\lambda_S \leq \lambda_{\min}(XSX^\top) \leq \lambda_{\max}(XSX^\top) \leq \Lambda_S
\]
\end{definition}

\section{Proposed Methods}
\label{sec:method}
\irls solves the robust regression problem by repeatedly alternating between the following two steps
\begin{enumerate}
	\item \textbf{Reweighing}: Given a model $\hvw$, assign every data point a weight $s_i$ inversely proportional to its residual w.r.t. $\hvw$ i.e. set $\vs_i = \frac1{\abs{\vx_i^\top\hvw - y_i}}$.
	\item \textbf{Weighted Least Squares}: Solve a weighted least squares problem $\min_\vw\sum_{i=1}^n\vs_i(y_i - \vx_i^\top\vw)^2$ with above weights to obtain a new model $\vwn = (XSX^\top)^{-1}XS\vy$ where $S = \diag(\vs)$.
\end{enumerate}

The intuition behind this procedure is that corrupted points are likely to suffer large residuals and hence get downweighted. Given that this procedure runs the risk of divide-by-zero errors and numerical precision issues, it is common to truncate weights by employing a \emph{truncation parameter} $M$ while assigning weights\footnote{\label{foot:trunc-reg-eq}Literature often cites a \emph{regularization} procedure that sets $\vs_i = \frac1{\max\bc{\abs{\vx_i^\top\hvw - y_i},\delta}}$ given a parameter $\delta$. Setting $\delta = \frac1M$ shows truncation to be equivalent to regularization.} to the points i.e. $\vs_i = \min\bc{\frac1{\abs{\vx_i^\top\hvw - y_i}},M}$. However, it is suboptimal to rely on any single truncation value $M$. To see why, take a hypothetical example where the adversary introduces corruptions using a \emph{fake model} $\tvw$ as $b_i = \vx_i^\top(\tvw - \vwo)$ (i.e. $y_i = \vx_i^\top\tvw$) for all $i \in B$.

\hspace*{\parindent} \emph{Situation 1}: If we set $M$ to a small value (aggressive truncation), then no data point can ever hope to get a large weight. However, convergence to $\vwo$ is assured only when points in $G$ receive really large weights in comparison to points in $B$. Setting a small value of $M$ thus prevents \irls from recovering $\vwo$ accurately.

\hspace*{\parindent} \emph{Situation 2}: If we always use a large value of $M$ (lax truncation) and are unlucky enough to initialize \irls close to $\tvw$, then points in the set $B$ will initially have very small residuals, hence receive large weights (which the large value of $M$ will allow) whereas points in the set $G$ will receive comparatively smaller weights. This will cause \irls to gravitate towards $\tvw$. This example precludes any hope of a global convergence guarantee and forces us to do careful initialization.

The above limitations of \irls are well corroborated by experiments (see \S\ref{sec:exps}). To remedy this, we propose the \stir algorithm in Algorithm~\ref{algo:stir}. \stir executes \irls, but in \emph{stages}, with initial stages employing aggressive truncation with a small value of $M$ and later stages successively relaxing the truncation.

The advantage of the above augmentation is that even if we have an unfortunate initialization, e.g. we start at $\tvw$ itself, the (initially) aggressive truncation will prevent bad points from getting large weights whereas good points, being in majority, even though receiving relatively smaller weights, will still prevent \stir from latching onto $\tvw$ and hopefully attract the procedure towards the gold model $\vwo$. Subsequent stages, where truncation is relaxed, will allow good points to be given large weights, thus differentiating them from bad points. This would force \stir towards $\vwo$.

Algorithm~\ref{algo:girls} presents \girls, a gradient version of \stir, that replaces weighted least squares by a much cheaper gradient step. This benefits large datasets, where solving weighted least squares repeatedly may be prohibitive. We note that although \emph{stagewise} \irls procedures have been proposed in literature \cite{BissantzDMS2009}, previous works neither give model recovery guarantees, nor offer scalable gradient versions of \irls.

\begin{algorithm}[t]
	\caption{\stir - Stagewise-Truncated \irls}
	\label{algo:stir}
	\begin{algorithmic}[1]
		{\small
		\REQUIRE Data $X, \vy$, initial truncation $M_1$, increment $\eta > 1$
		\ENSURE A model $\vw$
		\STATE $\vw^1 \leftarrow \vzero$
		\FOR{$T = 1, 2, \ldots, K-1$}
			\STATE $\vw^{T,1} \leftarrow \vw^T$
			\STATE $t \leftarrow 1$
			\WHILE{$\norm{\vw^{T,t+1}-\vw^{T,t}}_2 > \frac2{\eta M_T}$}
				\STATE $\vr^t \leftarrow X^\top\vw^{t,1} - \vy$
				\STATE $S^t \leftarrow \diag(\vs^t) ,\qquad \vs^t_i \leftarrow \min\bc{\frac1{\abs{\vr^t_i}},M_T}$
				\STATE $\vw^{T,t+1} \leftarrow (XS^tX^\top)^{-1}XS^t\vy$
				\STATE $t \leftarrow t+1$
			\ENDWHILE
			\STATE $\vw^{T+1} \leftarrow \vw^{T,t+1}$
			\STATE $M_{T+1} \leftarrow \eta\cdot M_T$
		\ENDFOR
		\STATE \textbf{return} {$\vw^K$}
		}
	\end{algorithmic}
\end{algorithm}

\begin{algorithm}[t]
	\caption{\girls: \stir-Gradient Descent}
	\label{algo:girls}
	\begin{algorithmic}[1]
		\makeatletter
		\setcounter{ALC@line}{7}
		\makeatother
		{\small
		\REQUIRE Data $X, \vy$, initial truncation $M_1$, increment $\eta > 1$, step length $C$
		\ENSURE A model $\vw$
				\STATE \hspace*{2em}$\vw^{T,t+1} \leftarrow \vw^{T,t} - \frac{2C}{M_Tn}\cdot XS^t\vr^t$\\
				\COMMENT{Rest of steps 1-14 remain same as in \stir}
		}
	\end{algorithmic}
\end{algorithm}

\section{\irls is Majorization-minimization on a Scaled Huber Loss}
\label{sec:huber}
Before presenting a convergence analysis for \stir, we point out a curious link between \irls, \stir and the Huber loss function. We note that our observation may be folklore. The Huber loss is widely used in robust regression applications \cite{AftabHartley2015, BissantzDMS2009, Cline1972, Osborne1985}, particularly those used in situations with heavy tailed noise.
\[
h_\epsilon(x) =
\begin{cases}
\frac12x^2 & \abs x \leq \epsilon\\
\epsilon\abs x - \frac12\epsilon^2 & \abs x > \epsilon
\end{cases}
\]
The function smoothly transitions from quadratic behavior close to the origin, to linear far from the origin. Now consider the following loss function
\[
f_\epsilon(x) =
\begin{cases}
\frac12\br{\frac{x^2}\epsilon + \epsilon} & \abs x \leq \epsilon\\
\abs x & \abs x > \epsilon
\end{cases}
\]
It is easily seen that $f_\epsilon(x) = \frac{h_\epsilon(x)}\epsilon + \frac\epsilon2$ and thus, $f_\epsilon()$ is simply a scaled (and translated) version of the Huber loss function, as well as that $\abs x \leq f_\epsilon(x) \leq \abs x + \frac\epsilon2$. Now, for any $a \in \bR, \epsilon > 0$, consider the function
\[
g_\epsilon(x;a) := \frac12\br{\frac{x^2}{\max\bc{\abs a,\epsilon}} + \max\bc{\abs a,\epsilon}}
\]
\begin{figure}%
\centering
\includegraphics[width=0.75\columnwidth]{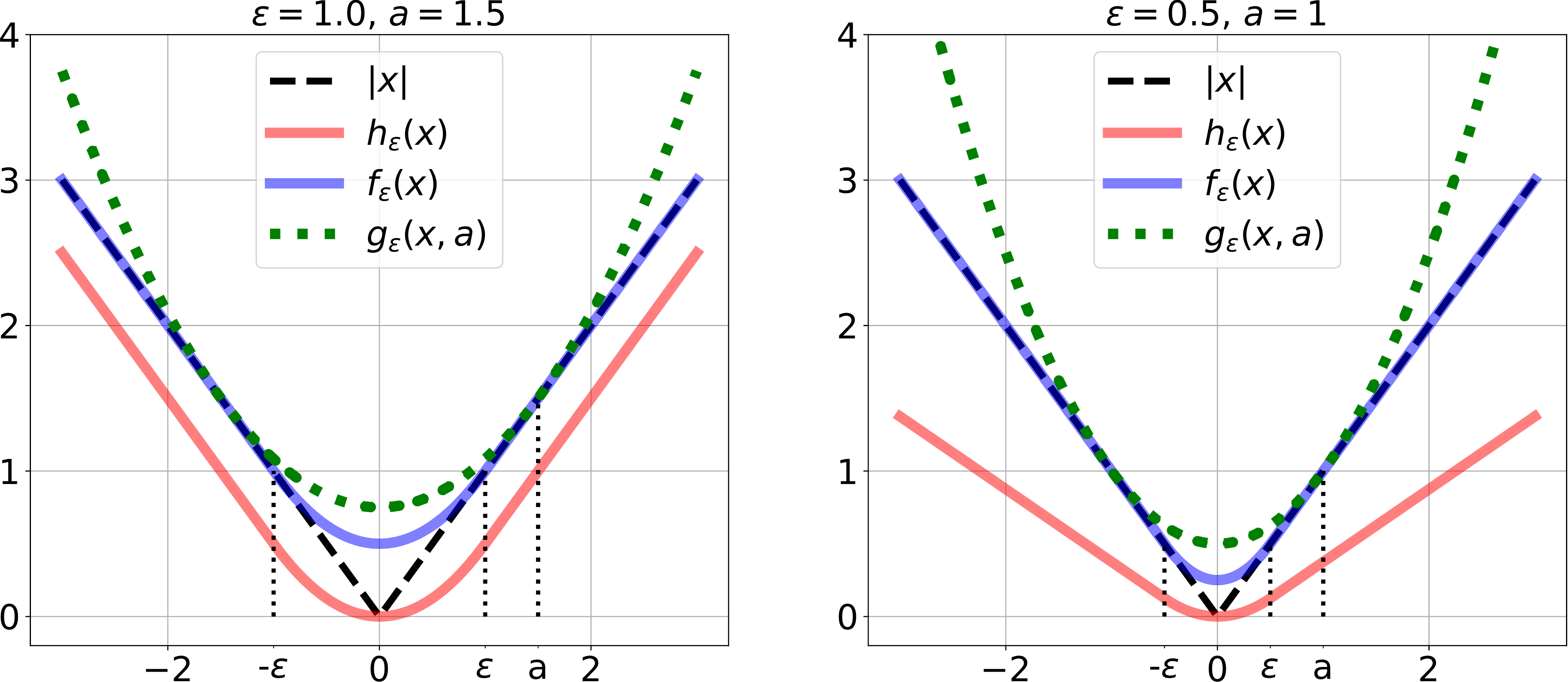}%
\caption{A depiction of Huber $h_\epsilon()$, scaled Huber $f_\epsilon()$ loss functions, and its majorizer $g_\epsilon()$ for various $\epsilon$.}%
\label{fig:huber}%
\end{figure}
Given a model $\vw^0$ and data $(\vx_i,y_i)_{i=1}^n$, denote
\begin{align*}
\ell_\epsilon(\vw) &:= \frac1n\sum_{i=1}^nf_\epsilon\br{\ip\vw{\vx_i}-y_i}\\
\wp_\epsilon(\vw;\vw^0) &:= \sum_{i=1}^ng_\epsilon\br{\ip\vw{\vx_i}-y_i;\ip{\vw^0}{\vx_i}-y_i}
\end{align*}
The following observations are key (see Appendix~\ref{app:huber}).
\begin{enumerate}
	\item $\wp_\epsilon(\cdot;\vw^0)$ is a majorizer for $\ell_\epsilon(\cdot)$ at $\vw^0, \forall \epsilon > 0$ i.e. $\wp_\epsilon(\vw;\vw^0) \geq \ell_\epsilon(\vw), \forall \vw$ but $\wp_\epsilon(\vw^0;\vw^0) = \ell_\epsilon(\vw^0)$
	\item If the current model is $\vw^0$ then $M$-truncated \irls minimizes $\wp_\frac1M(\vw;\vw^0)$ to obtain the next model.
	\item $\nabla\wp_\epsilon(\vw^0;\vw^0) = \nabla\ell_\epsilon(\vw^0)$.
\end{enumerate}

Thus, \irls can be seen as performing majorization-minimization \cite{Mairal2015} on the scaled Huber loss $\ell_\epsilon(\cdot)$. The reweighing step effectively constructs the majorizer function $\wp_\epsilon(\cdot,\vw^0)$ over which the least squares step then performs minimization. Point 3 above shows that \girls can be effectively seen as performing gradient descent with respect to $\ell_\epsilon(\vw^0)$.

This also allows us to interpret the stages of \stir as using scaled Huber losses with successively smaller values of $\epsilon$ (point 2 above shows that \stir sets $\epsilon = \frac1M$). Note that in the limit $\epsilon \rightarrow 0$, $\ell_\epsilon(\cdot)$ approaches the absolute error function, and thus, in the limit $M \rightarrow \infty$, \stir ends up optimizing the absolute error function. \girls can be seen as simply replacing the minimization steps with a gradient descent step.

\section{Convergence Analysis}
\label{sec:analysis}
In this section, we establish that both \stir and \girls enjoy a linear rate of convergence, as well as a breakdown point $\alpha \geq \Om1$. Theorem~\ref{thm:main} summarizes the results. It is notable that \stir and \girls offer a breakdown point of greater than $\frac1{5.25}$ (for Gaussian covariates -- see below for details), which is far superior to those offered by recent works such as \cite{BhatiaJK2015, BhatiaJKK2017} which offer breakdown points of $\approx \frac1{60}$ and $\frac1{10000}$ respectively (again for Gaussian covariates).

\begin{theorem}
\label{thm:main}
Suppose we have $n$ data points with the covariates $\vx_i$ sampled from a sub-Gaussian distribution $\cD$ and an $\alpha$ fraction of the data points are corrupted. If \stir (or \girls) is initialized at an (arbitrary) point $\vw^0$, with an initial truncation that satisfies $M_1 \leq \frac1{\norm{\vw^0-\vwo}_2}$, and executed with an increment $\eta > 1$ such that we have $\alpha \leq \frac c{2.88\eta + c}$, where $c > 0$ is a constant that depends only on $\cD$, then for any $\epsilon > 0$, with probability at least $1 - \exp(-\softOm n)$, after $K = \bigO{\log\frac1{M_1\epsilon}}$ stages, we must have $\norm{\vw^K - \vwo}_2 \leq \epsilon$. Moreover, each stage consists of only $\bigO1$ iterations.
\end{theorem}

\noindent\textbf{Global Convergence} Note that the above result allows initialization at any location $\vw^0$, so long as the accompanying value $M_1$ is small enough i.e. $M_1 \leq \frac1{\norm{\vw^0-\vwo}_2}$ which can be ensured using a simple binary search (see \S\ref{sec:exps} for details on parameter setting). In particular, if an estimated upper-bound $\norm{\vwo}_2 \leq W$ is available, then we can set $\vw^0 = \vzero$ and set $M_1 = \frac1W$.

Given this parameter convergence result, we can also establish that \stir and \girls offer linear convergence guarantees with respect to the Huber and absolute loss functions as well. We refer the reader to Appendix~\ref{app:huber-abs} for details.

\noindent\textbf{Breakdown Point} Both \stir and \girls enjoy a breakdown point of $\alpha \leq \frac c{2.88\eta + c}$ where $\eta$ is chosen by us and $c$ is a distribution dependent constant. Bounds on this constant are established for several interesting distributions in Appendix~\ref{app:c-values}. In particular, for the Gaussian distribution $\cN(\vzero, I_d)$, we have $c \geq 0.68$ which, for values of $\eta \rightarrow 1$, endow \stir and \girls with a breakdown point of greater than $\frac1{5.25}$.

\subsection{Proof Outline - the Peeling Strategy}
Given the stage-wise nature of our algorithms \stir and \girls, we employ a \emph{peeling}-based proof strategy that is a departure from the techniques used by previous results such as \cite{BhatiaJK2015,ChenD2012,WrightM2010}.

Our proof partitions the model space into \emph{annular peels} centered at the gold model $\vwo$ (see Figure~\ref{fig:peel}). The outermost peel has a radius of $\frac1{M_1}$, and successive inner peels have radii that are an $\eta$ factor smaller i.e. the subsequent peels have radii $\frac1{\eta M_1}, \frac1{\eta^2M_1}, \frac1{\eta^3M_1}, \ldots$. Note that by setting $M_1 \leq \frac1{\norm{\vw^0-\vwo}_2}$, \stir is guaranteed to reside inside the outermost peel in the beginning.

We then inductively show (see Lemmata~\ref{lem:induc-stir} and \ref{lem:induc-girls}) that once we are inside a certain peel, say $\norm{\vw - \vwo}_2 \leq \frac1{\eta^KM}$, and if the WSC/WSS properties hold with appropriate constants (see Appendix~\ref{app:wsc-wss}), then if we execute ($\eta^KM$)-truncated \irls for a constant number of iterations, we are guaranteed to obtain a model, say $\vw^+$, that ensures $\norm{\vw^+ - \vwo}_2 \leq \frac1{\eta^{K+1}M}$.

This implies that we have entered the next inner peel. We can now set the truncation level to $\eta^{K+1}M$ and continue the process. Note that this is exactly the algorithmic step performed by \stir/\girls (see Algorithm~\ref{algo:stir}, line 12) to start a new stage. Due to lack of space, all complete proofs are given in the appendices.

\begin{figure}%
\includegraphics[width=0.5\columnwidth]{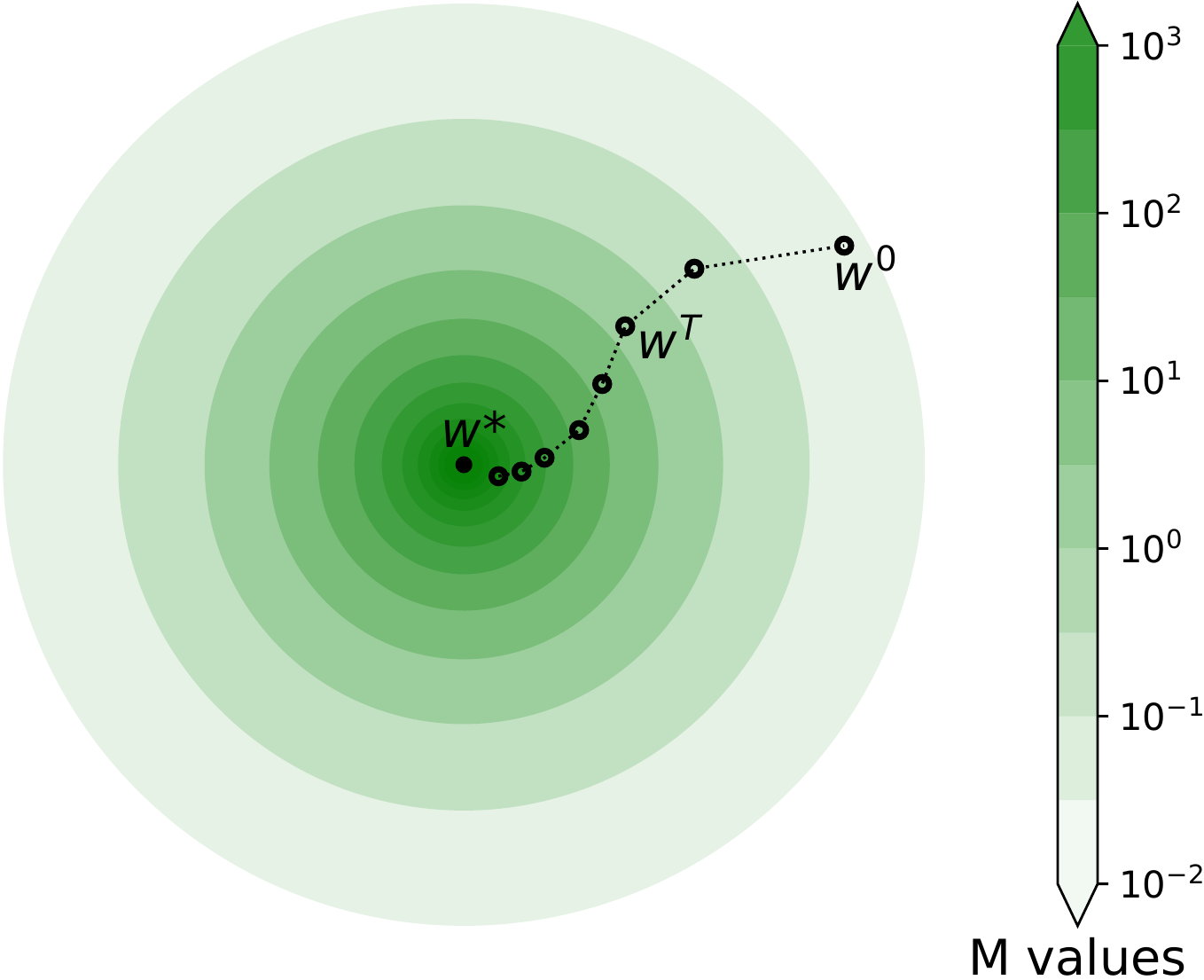}%
\centering
\caption{A depiction of the peeling process. The \stir procedure starts off far away from $\vwo$ and using a small value of $M$. In successive stages, it enters closer peels around $\vwo$ and also begins using larger values of $M$.}%
\label{fig:peel}%
\end{figure}

\subsection{Establishing WSC/WSS}
A central result required for the peeling strategy to work, is ensuring that our covariates satisfy the WSC/WSS properties (that we introduced in \S\ref{sec:formulation}) with respect to the weights assigned to data points by the \stir and \girls algorithms. We show that for covariates drawn from sub-Gaussian distributions, this is indeed true (see Appendix~\ref{app:wsc-wss}).

The use of such \emph{design properties} is quite common in literature e.g., restricted strong convexity/smoothness (RSC/RSS) \cite{DaubechiesDFG2010} in sparse recovery, and subset strong convexity/smoothness (SSC/SSS) \cite{BhatiaJK2015} in robust regression. It is also common to use results on extremal singular values of random matrices \cite{Vershynin2018}, to show that sub-Gaussian covariates satisfy RSC/RSS \cite{AgarwalNW2012} and SSC/SSS \cite{BhatiaJK2015}, with high probability.

However, doing so in our case is not as straightforward. The reason for this is that whereas the RSC/RSS and SSC/SSS properties are defined purely in terms of the data covariates, the WSC/WSS properties also incorporate data weights. Moreover, these weights are neither constant, nor independent of the data, but rather are assigned and repeatedly updated in a stage-wise manner by an algorithm such as \irls or \stir.

Since our proofs will require the WSC/WSS properties to hold with respect to \emph{all} weight assignments made during the entire execution of the algorithms, a direct application of classical techniques \cite{Vershynin2018} fails. Such techniques could have succeeded only if the data weights were to be constant or else independent of the data.

To overcome this challenge, we establish WSC/WSS properties for sub-Gaussian covariates in a \emph{peel-wise} manner using a careful uniform convergence bound. The number of peels is no more than $\bigO{\log\frac1\epsilon}$ since each peel corresponds to a stage of the algorithm and $\bigO{\log\frac1\epsilon}$ is the number of stages required to achieve an $\epsilon$-accurate solution (see Theorem~\ref{thm:main}), which then allows us to take a union bound over all peels.

Within each peel, a careful uniform convergence bound is employed over all models within that peel in order to establish WSC/WSS. Note that our results present a novel extension of the existing notions of SSC/SSS since we can recover SSC/SSS as a special case of WSC/WSS where the weights are simply zero or unity.

\subsection{Corruptions and Dense Noise}
\label{sec:dense}
So far we have looked at an idealized setting where the responses are either completely clean $y_i = \vx_i^\top\vwo$ for $i \in G$ or else corrupted $y_j = \vx_j^\top\vwo +\vb_j$ for $j \in B$. We now look at a more realistic setting where even the ``good'' points experience sub-Gaussian noise. We will now assume that our data is generated as $\vy = X^\top\vwo + \vb + \vepsilon$ where, as before $\norm\vb_0 \leq \alpha\cdot n$, but we additionally have $\vepsilon \sim \cD_\varepsilon$ where $\cD_\varepsilon$ is a $\sigma$-sub-Gaussian distribution with zero mean and real support \footnote{We can tolerate noise with non-zero mean as well, by using a simple pairing trick which has a side effect of at most doubling the corruption rate $\alpha$}.

We will denote $B := \supp(\vb)$ and $G := [n] \setminus B$, as before. Our covariates will continue to be sampled from an $R$-sub-Gaussian distribution $\cD$ with support over $\bR^d$. Even in this setting, we can ensure a model recovery result with a linear rate of convergence.

\begin{theorem}
\label{thm:dense}
Suppose we have $n$ data points with the covariates $\vx_i$ sampled from a sub-Gaussian distribution $\cD$ and an $\alpha$ fraction of the data points are corrupted with the rest subjected to sub-Gaussian noise sampled from a distribution $\cD_\varepsilon$ with sub-Gaussian norm $\sigma$. If \stir (or \girls) is initialized at an (arbitrary) point $\vw^0$, with an initial truncation that satisfies $M_1 \leq \frac1{\norm{\vw^0-\vwo}_2}$, and executed with an increment $\eta > 1$ such that we have $\alpha \leq \frac {c_\varepsilon}{5.85\eta + c_\varepsilon}$, where $c_\varepsilon > 0$ is a constant that depends only on the distributions $\cD$ and $\cD_\varepsilon$, then with probability at least $1 - \exp(-\softOm n)$, after $K = \bigO{\log\frac1{M_1\sigma}}$ stages, each of which has only $\bigO1$ iterations, we must have $\norm{\vw^K - \vwo}_2 \leq \bigO{\sigma}$.
\end{theorem}

We refer the reader to Appendix~\ref{app:noisy} for the full proof.

\noindent\textbf{Global Convergence} This result also allows arbitrary initialization so long as we set $M_1 \leq \frac1{\norm{\vw^0-\vwo}_2}$. However, note that this result only guarantees a convergence to $\norm{\vw^{K,1}- \vwo}_2 \leq \bigO\sigma$ and thus, does not ensure a consistent solution. We refer the reader to the proof of Theorem~\ref{thm:dense} in Appendix~\ref{app:noisy} for a discussion on this result. We also note that our results or our algorithms, do not require the knowledge of the noise parameter $\sigma$.

\noindent\textbf{Breakdown Point} For Gaussian covariates i.e. $\vx_i \sim \cN(\vzero, I_d)$, Gaussian noise i.e. $\vepsilon_i \sim \cN(0, \sigma^2)$, we have $c \geq 0.52$ (see Appendix~\ref{app:noisy}), and for $\eta \rightarrow 1$ this gives \stir and \girls with a breakdown point of $\frac1{12.25}$.

\begin{figure*}%
\centering
\begin{subfigure}[b]{0.45\textwidth}
	\includegraphics[width=\textwidth]{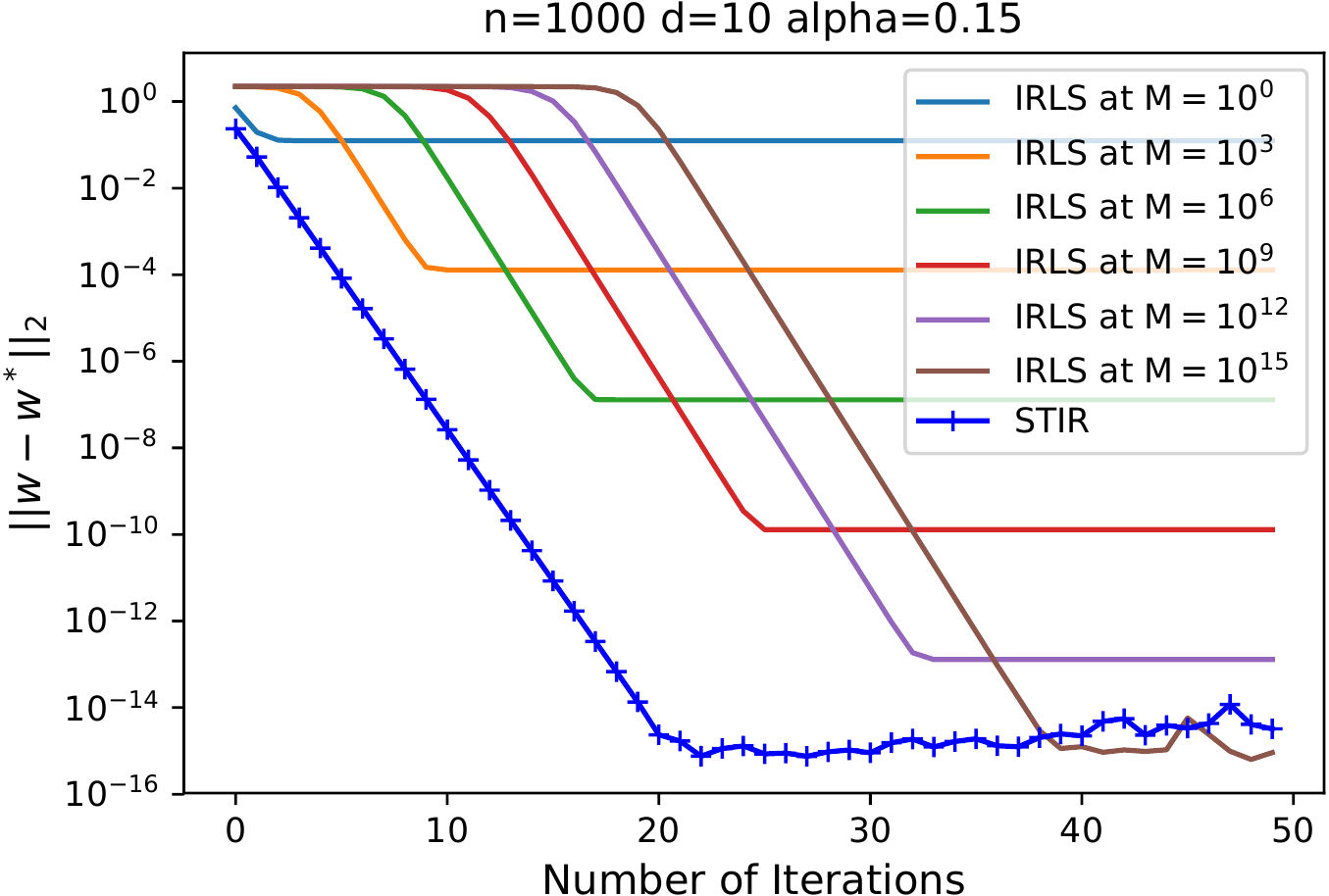}%
	\caption{\stir vs \irls with fixed M}
\end{subfigure}
\begin{subfigure}[b]{0.45\textwidth}
	\includegraphics[width=\textwidth]{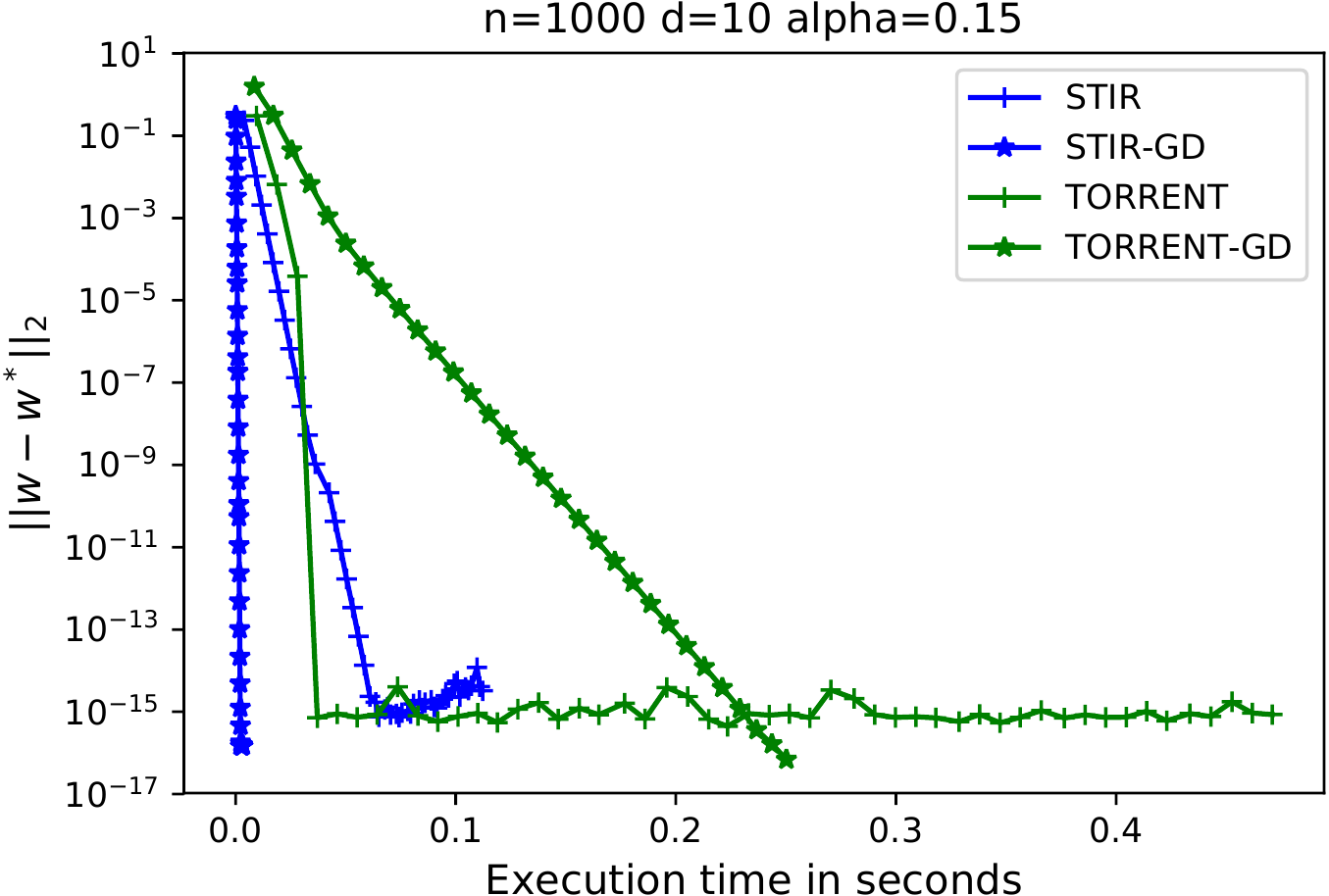}%
	\caption{\stir vs \torrent}
\end{subfigure}
\begin{subfigure}[b]{0.45\textwidth}
	\includegraphics[width=\textwidth]{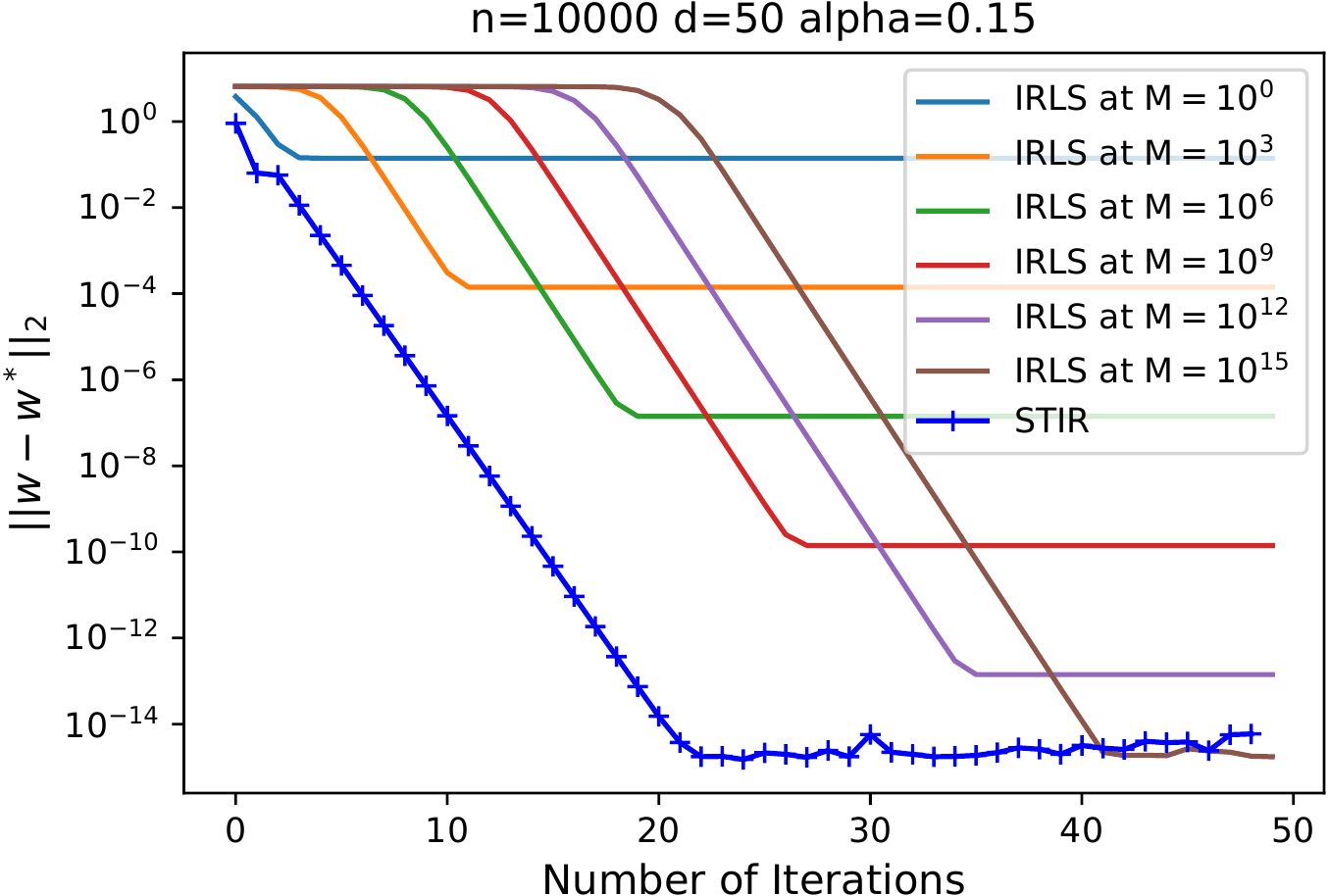}%
	\caption{\stir vs \irls with fixed M}
\end{subfigure}
\begin{subfigure}[b]{0.45\textwidth}
	\includegraphics[width=\textwidth]{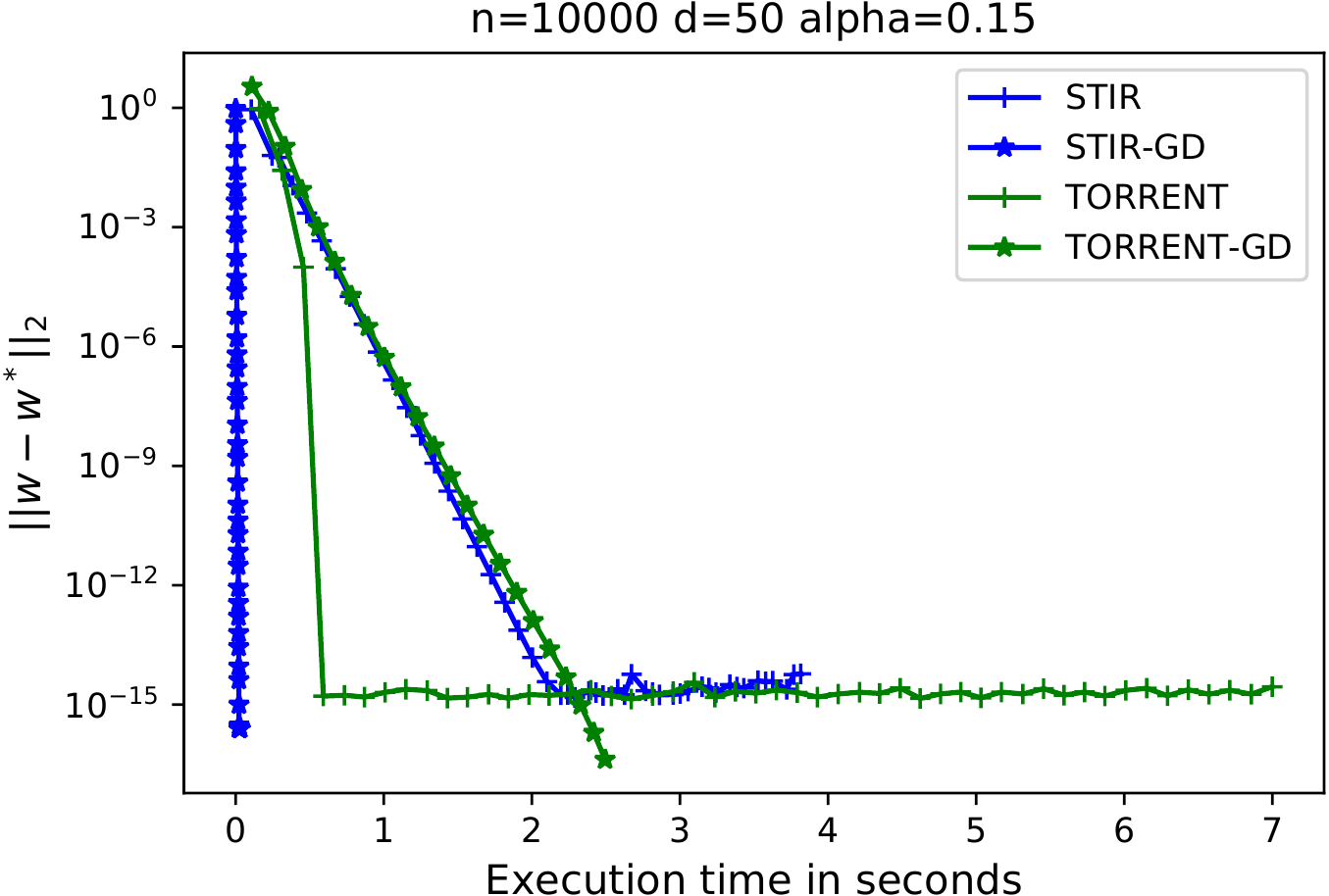}%
	\caption{\stir vs \torrent}
\end{subfigure}
\caption{All y-axes are in log-scale. Figs (a) and (c) use different data dimensionalities and number of data points and compare \stir to when \irls is executed with various fixed values of the truncation parameter $M$. It is clear that no fixed value performs well. For small fixed values $M \approx 10^0$, \irls converges rapidly but to poor models. For large fixed values $M \approx 10^{12}$, \irls gets stuck at the fake model and takes long to converge. On the other hand, although \stir was initialized with $M_1 = 0$ for this experiment, it adaptively increases its truncation parameter to offer far better convergence than \irls with any fixed value of $M$. Figs (b) and (d) compare \stir and \girls with \torrent and \tgd. In all cases, \girls offers the fastest convergence. 
}%
\label{fig:fig1}%
\end{figure*}

\begin{figure*}%
\centering
\begin{subfigure}[b]{0.45\textwidth}
	\includegraphics[width=\textwidth]{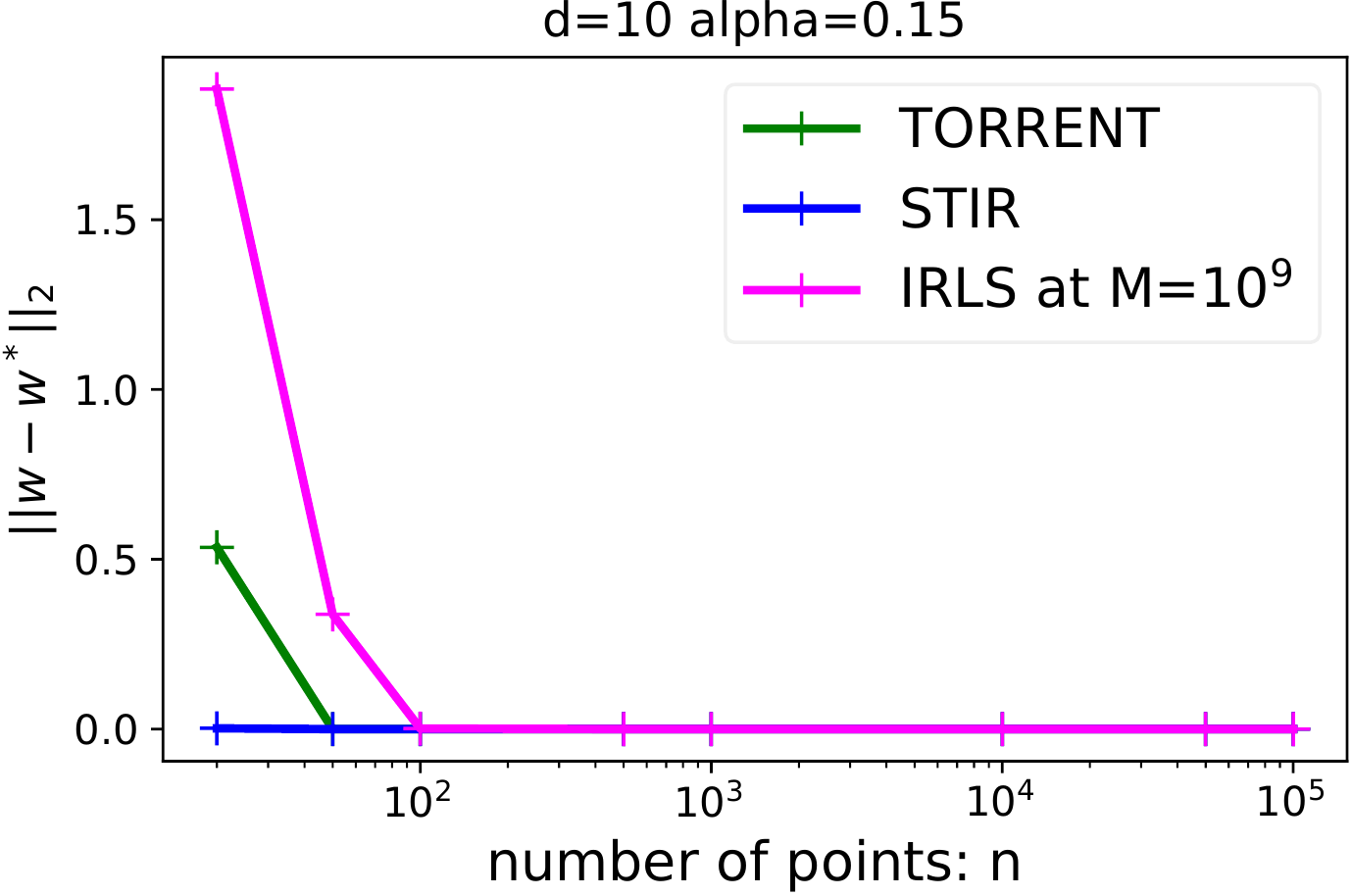}%
	\caption{Variation with dataset size}
\end{subfigure}
\begin{subfigure}[b]{0.45\textwidth}
	\includegraphics[width=\textwidth]{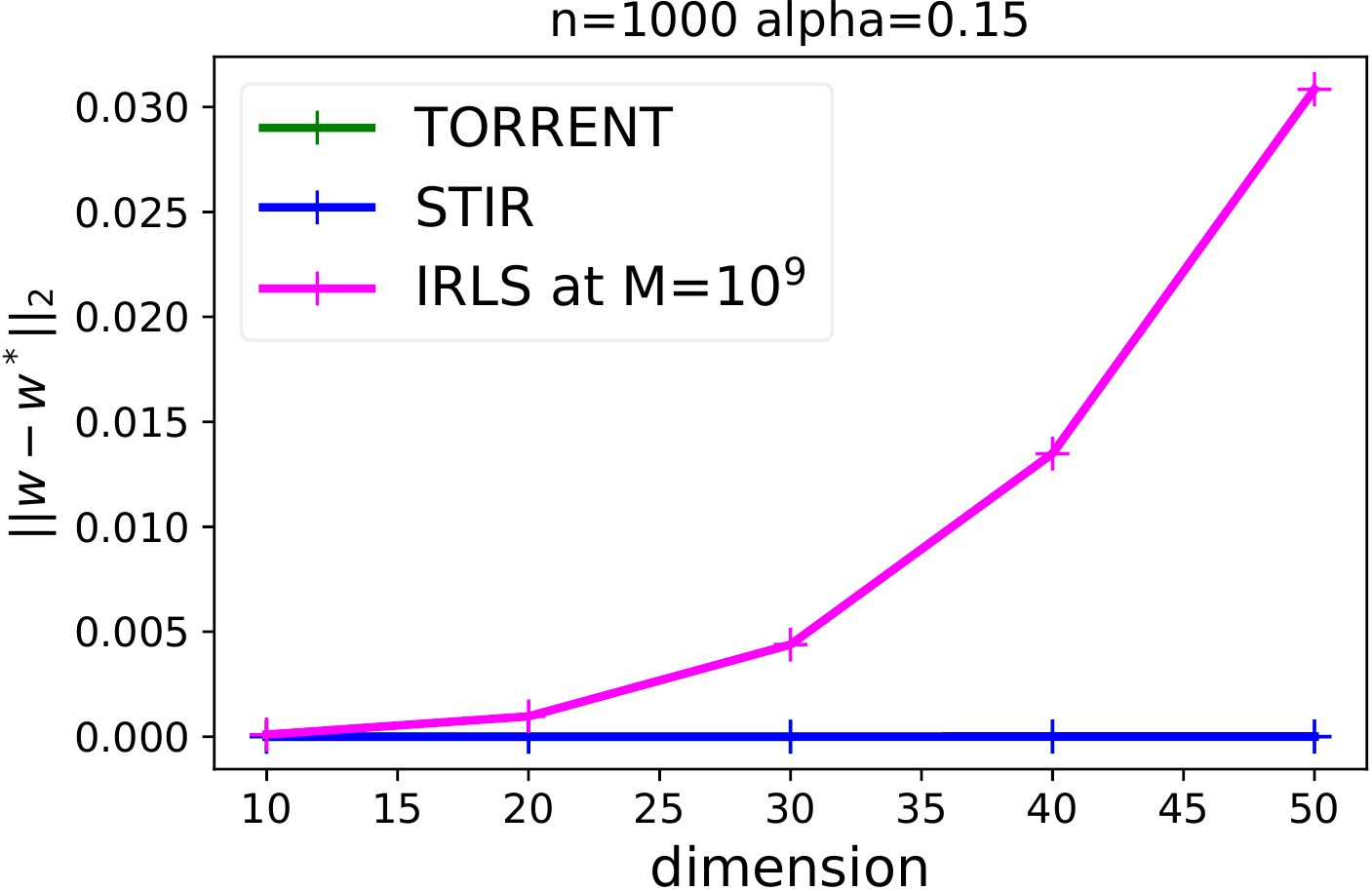}%
	\caption{Variation with dimension}
\end{subfigure}
\begin{subfigure}[b]{0.45\textwidth}
	\includegraphics[width=\textwidth]{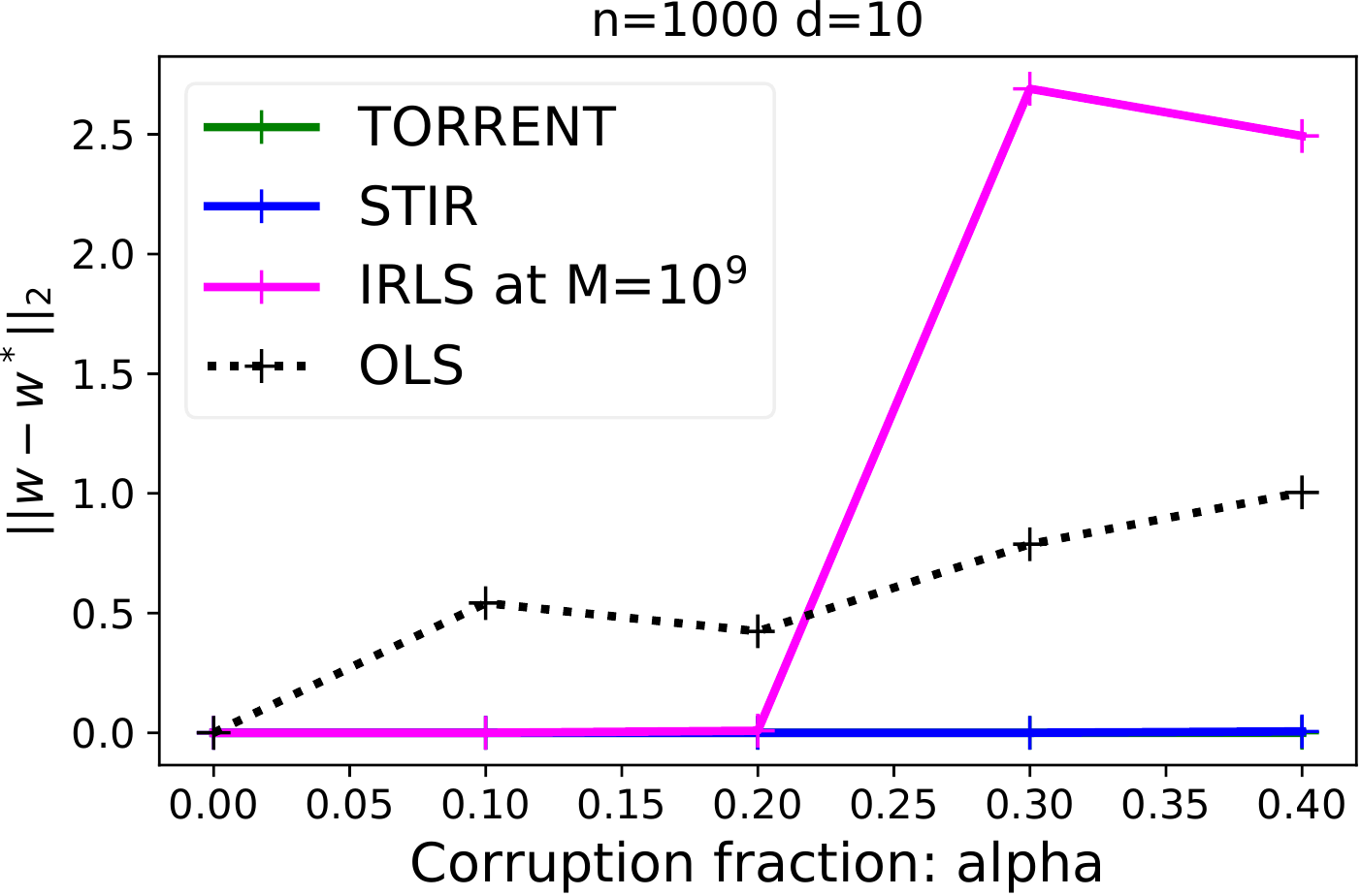}%
	\caption{Variation with corruption}
\end{subfigure}
\begin{subfigure}[b]{0.45\textwidth}
	\includegraphics[width=\textwidth]{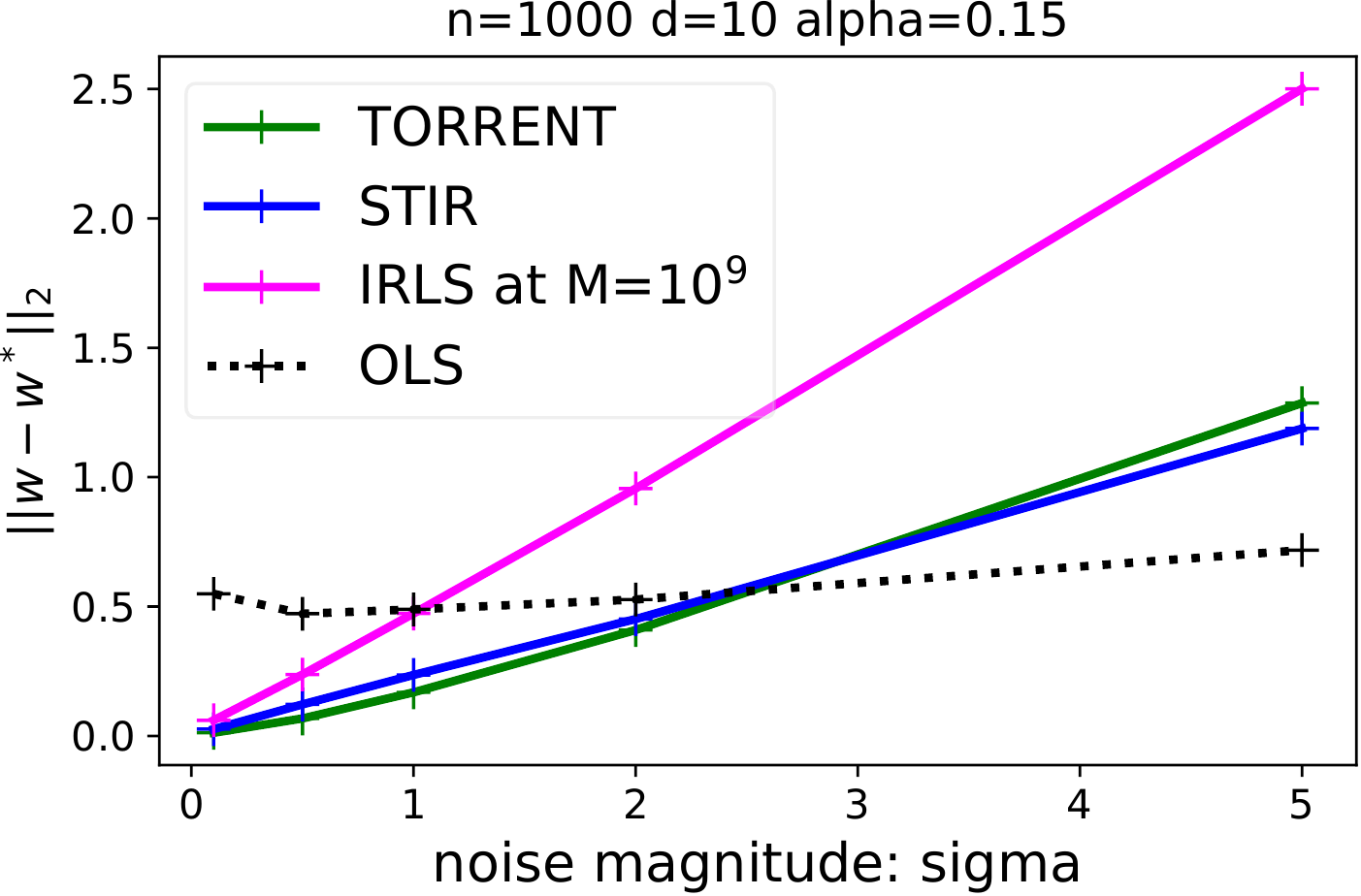}%
	\caption{Variation with white noise}
\end{subfigure}
\caption{The figures compare \stir, \torrent, \irls, and \ols for convergence behavior. \ols exceeds the figure boundaries and hence not visible in Figs (a) and (b). Fig (a) examines the effect of varying the training set size. Note that the x-axis is in log-scale. \irls performs poorly with very few data points but \stir and \torrent continue to offer good convergence. Fig (b) shows that \irls worsens with increasing dimensionality whereas \stir and \torrent remain stable. Fig (c) explores the affect of increasing the fraction of corrupted points. Both \ols and \irls show considerable worsening with increasing fraction of corruptions. Finally, Fig (d) explores the hybrid noise model discussed in Section~\ref{sec:dense} (Figs (a)-(c) had no white noise). Here, \irls performs the worst of all. However, once the noise variance goes beyond a point, \torrent and \stir start losing the distinction between good and bad points and the naive \ols starts outperforming them.}%
\label{fig:fig2}%
\end{figure*}

\begin{figure*}%
\centering
\begin{subfigure}[b]{0.45\textwidth}
	\includegraphics[width=\textwidth]{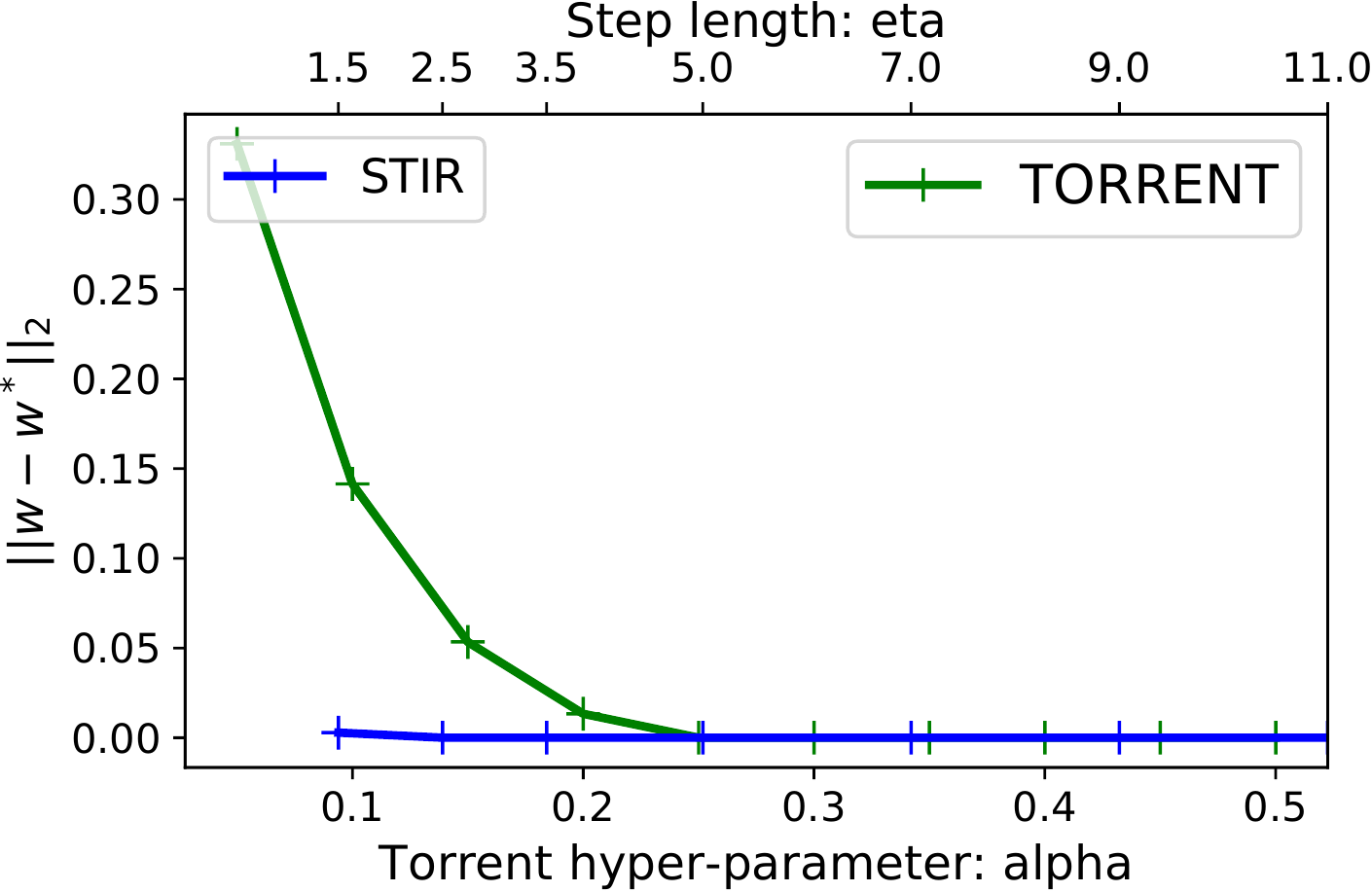}%
	\caption{Misspecifying parameters}
\end{subfigure}
\begin{subfigure}[b]{0.45\textwidth}
	\includegraphics[width=\textwidth]{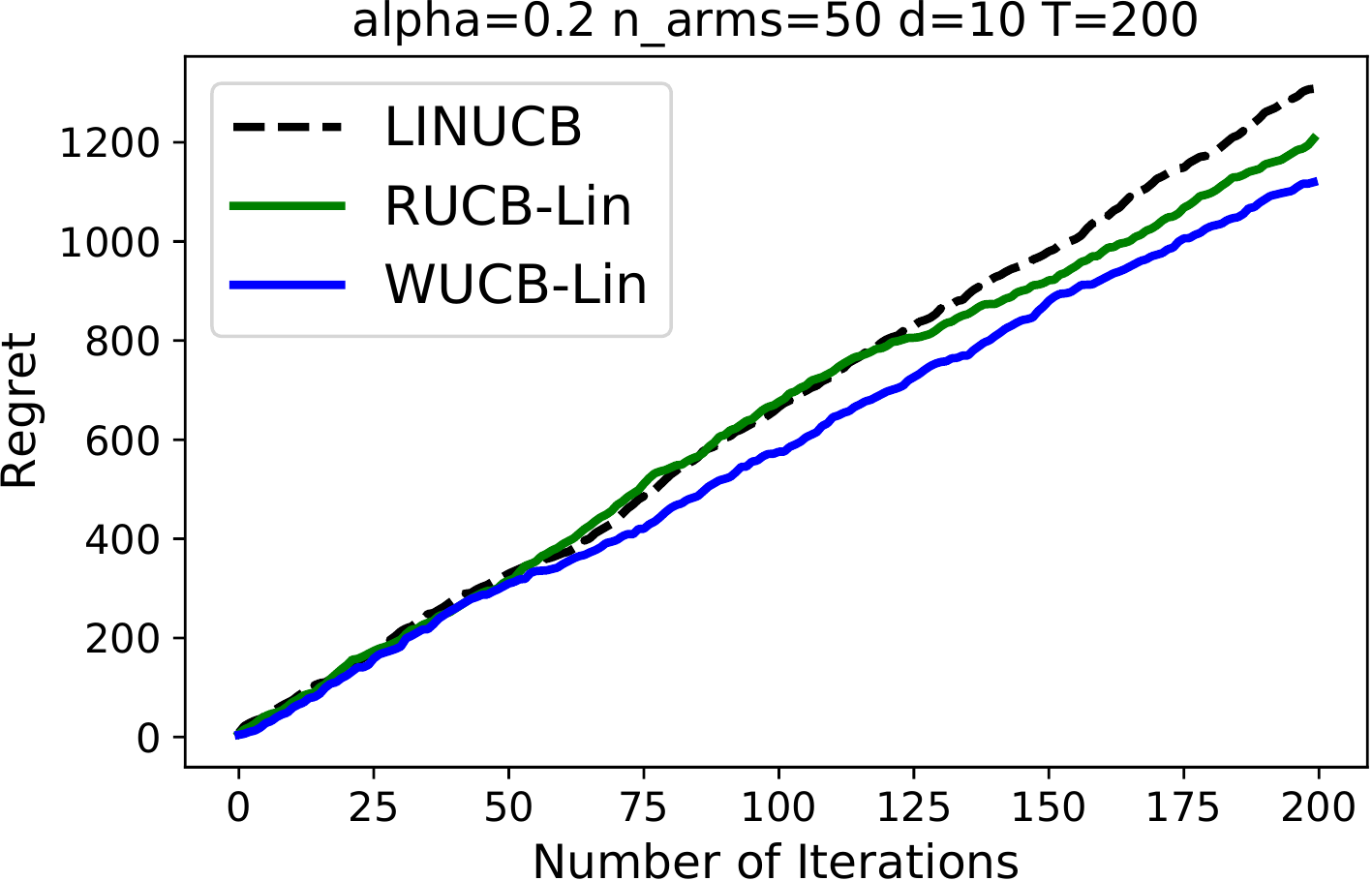}%
	\caption{Underreporting $\alpha$ as $0.15$}
\end{subfigure}
\begin{subfigure}[b]{0.45\textwidth}
	\includegraphics[width=\textwidth]{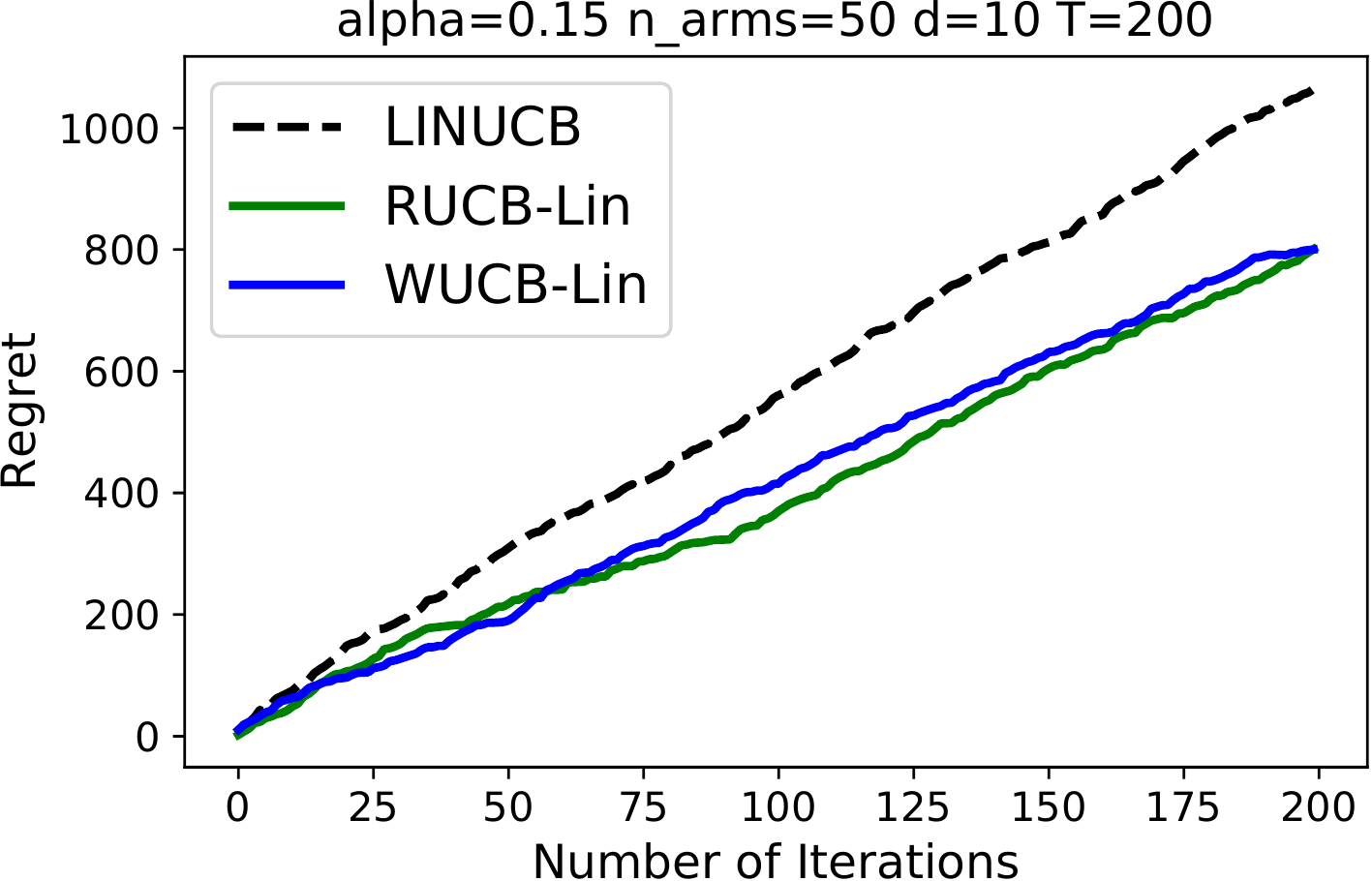}%
	\caption{Exact $\alpha$ to \torrent}
\end{subfigure}
\begin{subfigure}[b]{0.45\textwidth}
	\includegraphics[width=\textwidth]{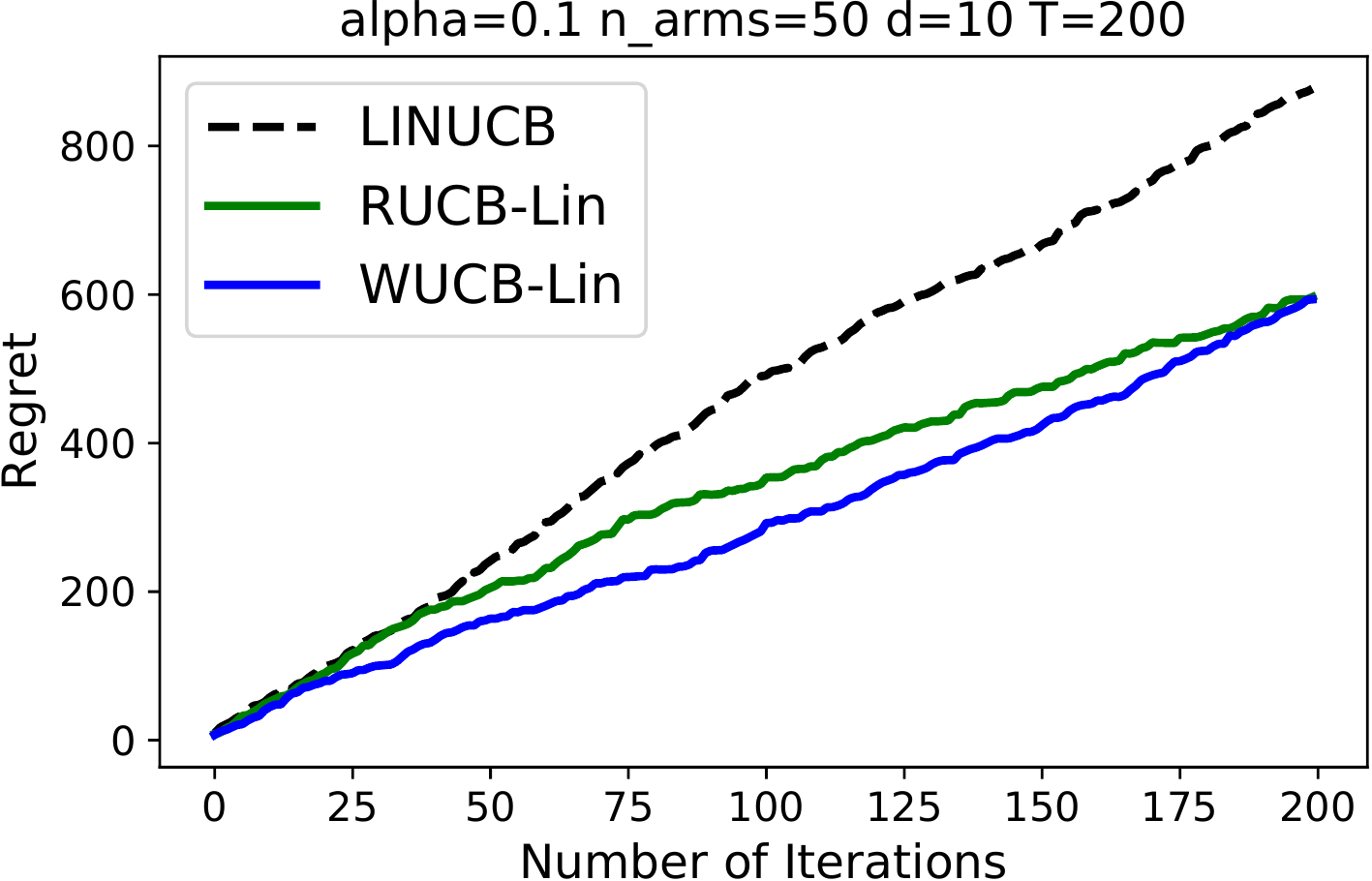}%
	\caption{Overreporting $\alpha$ as $0.15$}
\end{subfigure}
\caption{The figures compare \stir and \torrent with respect to hyperparameter misspecification. \stir was initialized at $\vw^0 = \vzero$ in these experiments. For Fig (a), 25\% data was corrupted but \torrent was given various values of its hyperparameter $\alpha$ (denoting the fraction of corrupted points) as indicated. \stir was also given various values of its own hyperparameter $\eta$ in a wide range. \torrent is very susceptible to hyperparameter misspecification and degrades heavily when not given a proper value whereas \stir is much more stable with respect to its hyperparameter. For Figs (b), (c), (d), respectively 20\%, 15\% and 10\% of the data was corrupted and linear-armed bandit algorithms that use \ols (LINUCB), \torrent (\rucb) and \stir (\wucb) were executed. For Figs (b), (c), (d), \torrent was always given a hyperparameter value $\alpha = 0.15$. Note that this is appropriate for Fig (c) where actually 15\% data was corrupted but not for Figs (b) and (d). \torrent performs comparably to \stir if provided the true value of $\alpha$, as in Fig (c) but its performance degrades if we give a value smaller than true value, such as in Fig (b) or a larger value, such as in Fig (d).}%
\label{fig:fig3}%
\end{figure*}

\section{Experiments}
\label{sec:exps}
In this section, we report results of a variety of experiments comparing \stir and \girls to other robust learning algorithms. These experiments were performed over two learning settings, namely robust linear regression and robust linear-armed bandit problems.

\noindent\textbf{Parameter and Adversary Setting} Algorithms considered in this section require only scalar parameters to be specified ($\alpha$ for \torrent, step length for \torrent-GD, $\eta$ and $M_1$ for \stir, and step length $C$ for \girls), all which were tuned via a fine grid search using a held-out validation set. In particular, a binary search was found to suffice for setting $M_1$. For all experiments, the adversary was made to introduce corruptions using a fake model as described in \S\ref{sec:method}. All algorithms were initialized at the fake model itself to test their behavior under adversarial initialization.

\subsection{Robust Regression Experiments}
We executed \stir and \girls on linear regression problems with response corruption as described in \S\ref{sec:formulation}.

\noindent\textbf{Algorithms}: We compared \stir and \girls with the \torrent algorithm \cite{BhatiaJK2015}, its faster gradient version \tgd, the classical \irls algorithm with various fixed values of the truncation parameter, and the standard \ols (Ordinary Least Squares) algorithm. We do not compare to some other state-of-the-art algorithms for robust regression, such as $L_1$ minimization techniques and extended Lasso since \cite{BhatiaJK2015} establishes that \torrent outperforms all of them.

\noindent\textbf{Data}: The covariate dimensionality and the number of data points are mentioned with each plot. All covariates were generated from a normal distribution. The gold and fake models were chosen as two independently sampled unit vectors. The set of ``bad'' data points was chosen randomly and the fake model was used to introduce corruptions, as in Section~\ref{sec:method}.

\subsection{Robust Linear Bandit Experiments}
As linear-armed bandit algorithms \cite{Abbasi-YadkoriPS2011} utilize regression routines internally, recent works have explored the possibility of using robust regression algorithms to target cases when arm-pulls are corrupted, for example \cite{KapoorPK2018} that uses \torrent itself to develop corruption-tolerant bandit learning algorithms.

Algorithm~\ref{algo:rucbl} presents \wucb, an adaptation of \stir to linear bandit settings. We refer the reader to Appendix~\ref{app:wucb} for details of the algorithm. \wucb roughly follows the popular \emph{Optimism-in-the-face-of-uncertainty} (OFUL) principle while selecting arms to pull at various time instants.

However, since we know some of the arm pulls generated corrupted rewards, instead of applying the OFUL principle blindly, \wucb invokes \stir and obtains not only an estimate of the reward generating model, but also a set of weights on previous arm pulls which indicate which pulls were corrupted and which pulls were clean. \wucb then uses these weights to form a \emph{weighted confidence set} (Algorithm~\ref{algo:rucbl}, line 6) that is further utilized in applying the OFUL principle to decide future arm pulls (Algorithm~\ref{algo:rucbl}, line 3).

\noindent\textbf{Algorithms and Data}: We compare \wucb with LINUCB that uses the simple \ols estimator, as well as the \rucb algorithm from \cite{KapoorPK2018}. We refer the reader to Appendix~\ref{app:wucb} for details of the problem setting.

\begin{algorithm}[t]
	\caption{\wucb: Weighted UCB for Linear Contextual Bandits}
	\label{algo:rucbl}
	\begin{algorithmic}[1]
			{\small
			\REQUIRE Upper bounds $\sigma_0$ (on sub-Gaussian norm of noise distribution), $B$ (on magnitude of corruption), $\alpha_0$ (on fraction of corrupted points), initial truncation $M_1$, increment rate $\eta$%
			\FOR{$t = 1,2,\dots,T$}
				\STATE Receive set of arms $A_t$
				\STATE Play arm $\hat{\vx}^t = \underset{{\vx \in A_t,\vw \in C_{t-1}}}{\arg\max}\ \ip{\vx}{\vw}$
				\STATE Receive reward $r_t$
				\STATE $(\hvw^t, S^t) \< $ \stir$\br{\bc{\hat\vx^\tau,r_\tau}_{\tau=1}^t, M_1, \eta}$\\
				\COMMENT{Denote $S^t = \diag(s^t_1, s^t_2, \ldots, s^t_t)$}
				\STATE $V^t \< \sum_{\tau \leq t}s^t_\tau\hat{\vx}^\tau(\hat{\vx}^\tau)^\top$, $X^t \leftarrow \bs{\hat{\vx}^1, \hat{\vx}^2, \ldots, \hat{\vx}^t}$
				\STATE $\bar\vw^t \< (V^t)^{-1}X^tS^t\vy$
				\STATE $C_t \< \{\vw: \norm{\vw - \bar\vw^t}_{V^t} \leq \sigma_0\sqrt{d\log T} + \alpha_0 BT\}$
			\ENDFOR
			}
	\end{algorithmic}
\end{algorithm}

\subsection{Discussion on Experiments}
Figures~\ref{fig:fig1}, \ref{fig:fig2} and \ref{fig:fig3} present graphs with the outcomes of the experiments. Although the respective captions in the figures detail the observed behaviours of various algorithms considered therein, here we point out some broad inferences.
\begin{enumerate}
	\item \girls offers much faster convergence as compared to \torrent or \tgd.
	\item No single value of the truncation parameter $M$ ensures a good performance with \irls. A stage-wise implementation with continuously updated truncation parameters, as \stir offers, is necessary for rapid and assuredly global convergence.
	\item \torrent requires an estimate of the fraction of corrupted points as a hyperparameter and is extremely susceptible to misspecification in this value. \stir and \girls on the other hand are much more resilient to misspecifications of their own hyperparameters.
\end{enumerate}

\section{Conclusion and Future Work}
In this work we presented \stir, a stage-wise algorithm that makes simple and efficient modifications, including a gradient-based implementation \girls, to the well-known \irls heuristic to obtain the first global convergence results for robust regression. These algorithms offer not only theoretically superior results to state-of-the-art algorithms such as \torrent but are empirically faster and more immune to hyperparameter mis-specification.

Our theoretical results are superior to those of previous works in terms of offering a better breakdown point, and are based on a novel notion of weighted strong convexity. Working with this new notion of strong convexity required us to develop the peeling proof technique which is novel in robust regression literature and may be of independent interest in analyzing other iterative algorithms.

Several avenues of future work exist. It would be interesting to examine other weighing functions (\irls and \stir use the inverse of the residual) for robust regression. It is likely that any reasonable decreasing function of residuals should suffice. It would also be interesting to derive formal regret bounds for the \wucb algorithm and see how they compare to the regret bounds of the \rucb algorithm from \cite{KapoorPK2018}.

\section*{Acknowledgements}
The authors thank the anonymous reviewers for several helpful comments on a previous version of the paper.
B.M. thanks the Research-I Foundation for a travel grant. P.K. is supported by the Deep Singh and Daljeet Kaur Faculty Fellowship and the Research-I Foundation at IIT Kanpur, and thanks Microsoft Research India and Tower Research for research grants.

\bibliographystyle{plain}
\bibliography{refs}

\allowdisplaybreaks
\appendix

\section{IRLS and the Scaled Huber Loss - Supplementary Details}
\label{app:huber}

We recapitulate below the definitions of the Huber loss, the scaled (and translated) Huber loss and, given a model $\vw^0$ and data $(\vx_i,y_i)_{i=1}^n$, other allied functions.
\begin{align*}
h_\epsilon(x) &=
\begin{cases}
\frac12x^2 & \abs x \leq \epsilon\\
\epsilon\abs x - \frac12\epsilon^2 & \abs x > \epsilon
\end{cases}\\
f_\epsilon(x) &=
\begin{cases}
\frac12\br{\frac{x^2}\epsilon + \epsilon} & \abs x \leq \epsilon\\
\abs x & \abs x > \epsilon
\end{cases}\\
g_\epsilon(x;a) &:= \frac12\br{\frac{x^2}{\max\bc{\abs a,\epsilon}} + \max\bc{\abs a,\epsilon}}\\
\ell_\epsilon(\vw) &:= \frac1n\sum_{i=1}^nf_\epsilon\br{\ip\vw{\vx_i}-y_i}\\
\wp_\epsilon(\vw;\vw^0) &:= \sum_{i=1}^ng_\epsilon\br{\ip\vw{\vx_i}-y_i;\ip{\vw^0}{\vx_i}-y_i}
\end{align*}
The claim that $M$-truncated \irls minimizes $\wp_\frac1M(\vw;\vw^0)$ to obtain the next model can be easily verified using the equivalence between the truncation and regularization techniques explained in Footnote~\ref{foot:trunc-reg-eq} (see \S\ref{sec:method} for the footnote). In the following, we establish that $g_\epsilon(\cdot;\cdot)$ is a valid majorizer for $f_\epsilon$ for any $\epsilon > 0$.
\begin{claim}
For any $a,x \in \bR, \epsilon > 0$, we have $g_\epsilon(a;a) = f_\epsilon(a)$ as well as $g_\epsilon(x;a) \geq f_\epsilon(x)$.
\end{claim}
\begin{proof}
We have, for the first claim,
\[
g_\epsilon(a;a) = \frac12\br{\frac{a^2}{\max\bc{\abs a,\epsilon}} + \max\bc{\abs a,\epsilon}} =
\begin{cases}
\frac12\br{\frac{a^2}\epsilon + \epsilon} & \abs a \leq \epsilon\\
\abs a & \abs a > \epsilon
\end{cases}
= f_\epsilon(a).
\]
For the second claim, we consider two simple cases
\begin{description}
	\item[Case 1 $\abs x > \epsilon$]: In this case we have $f_\epsilon(x) = \abs x$ and we always have $\frac12\br{\frac{x^2}{\max\bc{\abs a,\epsilon}} + \max\bc{\abs a,\epsilon}} \geq \abs x$.
	\item[Case 2 $\abs x \leq \epsilon$]: In this case denote $b = \max\bc{\abs a,\epsilon}$. Then we have $b \geq \epsilon \geq \abs x$ which gives us $x^2 \leq b\epsilon$. Thus, we have $g_\epsilon(x;a) - f_\epsilon(x) = \frac12\br{\frac{x^2}b + b} - \frac12\br{\frac{x^2}\epsilon + \epsilon} = \frac{(b - \epsilon)(b\epsilon - x^2)}{2b\epsilon} \geq 0$.\qedhere
\end{description}
\end{proof}

The following claim shows that we have $\left.f'_\epsilon(x)\right|_{x=a} = \left.g'_\epsilon(x;a)\right|_{x=a}$ for any $\epsilon, a$. This immediately establishes that $\nabla\wp_\epsilon(\vw^0;\vw^0) = \nabla\ell_\epsilon(\vw^0)$ for any model $\vw^0$.
\begin{claim}
For any $a,x \in \bR, \epsilon > 0$, we have $\left.f'_\epsilon(x)\right|_{x=a} = \left.g'_\epsilon(x;a)\right|_{x=a}$.
\end{claim}
\begin{proof}
We have $g'_\epsilon(x;a) = \frac{x}{\max\bc{\abs a,\epsilon}}$ which gives us
\[
\left.g'_\epsilon(x;a)\right|_{x=a} =
\begin{cases}
\frac a\epsilon & \abs a \leq \epsilon\\
\sign(a) & \abs a > \epsilon,
\end{cases}
\]
whereas we have
\[
f'_\epsilon(x) =
\begin{cases}
\frac x\epsilon & \abs x \leq \epsilon\\
\sign(x) & \abs x > \epsilon
\end{cases},
\]
which establishes the claim.
\end{proof}

\section{Supporting Results}
\label{app:supp}
In this section we prove a few results used in the convergence analysis of \stir.
\begin{lemma}
\label{lem:ssc-sss}
Suppose we have data covariates $X = \bs{\vx_1,\ldots,\vx_n}$ generated from an isotropic but otherwise arbitrary sub-Gaussian distribution. Then for any fixed set $S \subset [n]$ and $n = \Om{d+\log\frac1\delta}$, with probability at least $1-\delta$,
\[
0.99\abs{S} \leq \lambda_{\min}(X_SX_S^\top) \leq \lambda_{\max}(X_SX_S^\top) \leq 1.01\abs{S},
\]
where the constant inside $\Om{\cdot}$ depends only on the sub-Gaussian distribution and universal constants.
\end{lemma}
\begin{proof}
This is a special case of \cite[Lemma 16]{BhatiaJK2015} for isotropic distributions. Note that since our adversary is partially adaptive, the sets of good and bad points $G,B$ are fixed and this lemma applies to both $G$ and $B$.
\end{proof}

\begin{lemma}
\label{lem:rx}
Suppose our data covariates $\vx_1,\ldots,\vx_n$ are generated from a sub-Gaussian distribution with sub-Gaussian norm $R$. Then with probability at least $1-\delta$, we have $R_X := \max_{i \in [n]}\ \norm{\vx_i}_2 \leq \norm{\vmu}_2 + \bigO{R\sqrt{d + \log\frac n\delta}}$.
\end{lemma}
\begin{proof}
If $\vx$ is $R$-sub-Gaussian with mean $\vmu$, then for any unit vector $\vv \in S^{d-1}$, $\ip{\vv}{\vx-\vmu}$ is centered as well as $2R$-sub-Gaussian which gives us
\[
\P{\abs{\ip{\vv}{\vx-\vmu}} \geq t} \leq 2\exp\bs{-t^2/2R^2}
\]
If $\vv^1,\vv^2 \in S^{d-1}$, such that $\norm{\vv^1-\vv^2}_2 \leq \frac12$, then we have $\abs{\ip{\vv^1-\vv^2}{\vx-\vmu}} \leq \frac12\cdot\norm{\vx-\vmu}_2$. Thus, taking a union bound over a $1/2$-net over $S^{d-1}$ gives us
\[
\P{\max_{\vv \in S^{d-1}}\abs{\ip{\vv}{\vx-\vmu}} \geq \frac12\cdot\norm{\vx-\vmu}_2 + t} = \P{\norm{\vx}_2 \geq \norm{\vmu}_2 + 2t} \leq 2\cdot5^d\exp\bs{-t^2/2R^2}
\]
Taking $t^2 = 2R^2(d\log5 + \log\frac n\delta + \log 2)$ proves the result.
\[
\P{\max_{i \in [n]}\ \norm{\vx_i}_2 > \norm{\vmu}_2 + R\sqrt{2\br{d\log 5 + \log \frac n\delta + \log2}}} \leq \delta\qedhere
\]
\end{proof}

In the following, we establish that the scaled Huber loss is Lipschitz. This will be helpful in transferring our convergence guarantees to those with respect to the Huber and absolute loss functions.
\begin{lemma}
\label{lem:feps-lip}
For any $\epsilon > 0$, we have $\abs{\ell_\epsilon(\vw)-\ell_\epsilon(\vw')} \leq \norm{\vw-\vw'}_2\cdot\sqrt{1.01}$.
\end{lemma}
\begin{proof}
The function $f_\epsilon(\cdot)$ is clearly $1$-Lipschitz for any $\epsilon > 0$. This means that we have
\begin{align*}
\abs{\ell_\epsilon(\vw)-\ell_\epsilon(\vw')} &\leq \frac1n\sum_{i=1}^n\abs{\ip\vw{\vx_i}-\ip{\vw'}{\vx_i}} = \frac1n\norm{X^\top(\vw-\vw')}_1 \leq \frac1{\sqrt n}\norm{X^\top(\vw-\vw')}_2\\
&\leq \frac1{\sqrt n}\norm X_2\norm{\vw-\vw'}_2 \leq \norm{\vw-\vw'}_2\cdot\sqrt{1.01},
\end{align*}
where the last step follows due to Lemma~\ref{lem:ssc-sss}.
\end{proof}

\section{Convergence Analysis - Supplementary Details}
\label{app:conv}

We begin by restating Theorem~\ref{thm:main}, the main result that we will prove in this section.

\begin{reptheorem}{thm:main}
Suppose we have $n$ data points with the covariates $\vx_i$ sampled from a sub-Gaussian distribution $\cD$ and an $\alpha$ fraction of the data points are corrupted. If \stir (or \girls) is initialized at an (arbitrary) point $\vw^0$, with an initial truncation that satisfies $M_1 \leq \frac1{\norm{\vw^0-\vwo}_2}$, and executed with an increment $\eta > 1$ such that we have $\alpha \leq \frac c{2.88\eta + c}$, where $c > 0$ is a constant that depends only on $\cD$, then for any $\epsilon > 0$, with probability at least $1 - \exp\br{-\Om{n - d\log(d + n) + \log\frac1{M_1\epsilon}}}$, after $K = \bigO{\log\frac1{M_1\epsilon}}$ stages, we must have $\norm{\vw^K - \vwo}_2 \leq \epsilon$. Moreover, each stage consists of only $\bigO1$ iterations.
\end{reptheorem}
\begin{proof}
As mentioned before, notice that this is indeed a global convergence guarantee since it places no restrictions on the initial model $\vw^0$. The only requirement is that the accompanying initial truncation parameter $M_1$ complement the model initialization by satisfying $M_1 \leq \frac1{\norm{\vw^0-\vwo}_2}$. In particular, if initialized at the origin, as Algorithms~\ref{algo:stir} and \ref{algo:girls} do, we need only ensure $M_1 \leq \frac1{R_W}$ where $R_W = \norm{\vwo}_2$. This can be done using a simple binary search to identify an appropriate value of $M_1$. Recall that both \stir and \girls operate in stages. We introduce a notion of a \emph{well-initialized} stage below.
\begin{definition}[Well-initialized Stage]
A stage in the execution of \stir or \girls is said to be well-initialized if, given the truncation parameter $M_T$ which will be used during that stage, at the beginning of that stage $T$, we are in possession of a model $\vw^{T,1}$ that satisfies $\norm{\vw^{T,1}-\vwo}_2 \leq \frac1{M_T}$.
\end{definition}
Note that the initialization of \stir and \girls with respect to the setting of $M_1$ ensure $M_1 \leq \frac1{\norm{\vw^0-\vwo}_2}$ which implies that the very first stage is always well-initialized. Now, Lemmata~\ref{lem:induc-stir} and \ref{lem:induc-girls} show that, if the preconditions of this theorem are satisfied, then a stage $T$, started off with a model $\vw^T =: \vw^{T,1}$ (see Algorithm~\ref{algo:stir}, line 3) and a truncation parameter $M_T$ that satisfy the well-initialized condition i.e. $\norm{\vw^{T,1}-\vwo}_2 \leq \frac1{M_T}$, will ensure with probability at least $1 - \exp\br{-\Om{n - d\log(d + n)}}$, that there exists an upper bound of $t_0 = \bigO1$ iterations, such that we are assured that $\norm{\vw^{T,\tau} - \vwo}_2 \leq \frac1{\eta M_T}$ for all $\tau \geq t_0$.

An application of the triangle inequality shows that we will have $\norm{\vw^{T,t_0} - \vw^{T,t_0+1}}_2 \leq \frac2{\eta M_T}$ which implies (see Algorithm~\ref{algo:stir}, line 5) that we will exit this stage at the $(t_0+1)\nth$ inner iteration. However, notice that at this point we are endowed with $\norm{\vw^{T+1,1} - \vwo}_2 = \norm{\vw^{T+1} - \vwo}_2 =  \norm{\vw^{T,t_0+1} - \vwo}_2 \leq \frac1{\eta M_T} = \frac1{M_{T+1}}$. Note that this means that stage $(T+1)$ is well-initialized too.

Thus, whenever a stage $T$ is well-initialized, with probability at least $1 - \exp\br{-\Om{n - d\log(d + n)}}$, we have $\norm{\vw^{T+1,1} - \vwo}_2 \leq \frac1\eta\norm{\vw^{T,1}-\vwo}_2$. Since we always set $\eta > 1$, there exists an upper bound $T_0 = \bigO{\log\frac1{M_1\epsilon}}$ on the number of stages. Thus, an application of union bound shows that we must have $\norm{\vw^{T_0+1,1} - \vwo}_2 \leq \epsilon$ with probability at least $1 - \exp\br{-\Om{n - d\log(d + n)} + \log\frac1{M_1\epsilon}} = 1 - \exp(-\softOm n)$ for all $\epsilon = \frac1{n^{\bigO1}}$.
\end{proof}

\begin{lemma}
\label{lem:induc-stir}
Suppose we have $n$ data points with the covariates $\vx_i$ sampled from a sub-Gaussian distribution $\cD$ and an $\alpha$ fraction of the data points are corrupted. Suppose we initialize a stage $T$ within an execution of \stir with truncation level $M$, increment parameter $\eta$, and a model $\vw^T =: \vw^{T,1}$ such that $\alpha \leq \frac c{2.88\eta + c}$ and $\norm{\vw - \vwo}_2 \leq \frac1M$, then with probability at least $1 - \exp\br{-\Om{n - d\log(d + n)}}$, there exists an upper bound of $t_0 = \bigO1$ iterations, such that we are assured that $\norm{\vw^{T,\tau} - \vwo}_2 \leq \frac1{\eta M}$ for all $\tau \geq t_0$. Here $c$ is the constant of the WSC property and depends only on the distribution $\cD$ (see Lemma~\ref{lem:wsc-wss}).
\end{lemma}
\begin{proof}
Let $\vw^{T,\tau}$ be a model encountered by \stir within this stage and let $\vr = X^\top\vw^{T,\tau} - \vy$ denote the residuals due to $\vw^{T,\tau}$ and $S = \diag(\vs)$ denote the diagonal matrix of weights where $\vs_i = \min\bc{\frac1{\abs{\vr_i}},M}$. Then \stir will choose as the next model $\vw^{T,\tau+1} = (XSX^\top)^{-1}XS\vy = \vwo + (XSX^\top)^{-1}XS\vb$ which gives us
\[
\norm{\vw^{T,\tau+1} - \vwo}_2 \leq \frac{\norm{XS\vb}_2}{\lambda_{\min}(XSX^\top)}
\]
Now by Lemma~\ref{lem:ssc-sss}, with probability at least $1 - \exp(-\Om{n - d})$, we have $\norm{X_B}_2 = \sqrt{\lambda_{\max}(X_BX_B^\top)} \leq \sqrt{1.01B}$. By Lemma~\ref{lem:bad} we have, again with probability at least $1 - \exp(-\Om{n - d})$
\[
\norm{S\vb}_2 \leq \sqrt{4B(1+1.01M^2\norm{\vw - \vwo}_2^2)} \leq 2\sqrt{2.01B}
\]
It should be noted that Lemma~\ref{lem:bad} relies precisely on Lemma~\ref{lem:ssc-sss} to derive its confidence assurance. Since the nature of Lemma~\ref{lem:ssc-sss} is such that it need be established only once, and not repeatedly for every iteration, we have, with probability at least $1 - \exp(-\Om{n - d})$, for \emph{all iterations} within this stage (actually all iterations across all stages), both Lemma~\ref{lem:bad} and Lemma~\ref{lem:ssc-sss} hold simultaneously.

Using Lemma~\ref{lem:wsc-wss}, with probability at least $1 - \exp\br{-\Om{n - d\log(d + n)}}$, we have $\lambda_{\min}(XSX^\top) \geq \lambda_{\min}(X_GS_GX_G^\top) \geq 0.99c\cdot GM$. Note that since all models $\vw^{T,\tau}, \tau \geq 1$ in this stage will at least satisfy $\norm{\vw^{T,\tau} - \vwo}_2 \leq \frac1M$ (since the initial model $\vw^{T,1}$ satisfies this by assumption and \stir offers monotonic convergence), the result of Lemma~\ref{lem:wsc-wss} applies uniformly to all these models and need not be applied separately to each model in this stage. Using these results to upper bound $\norm{XS\vb}_2$ and lower bound $\lambda_{\min}(XSX^\top)$ shows that at either we must have 
\[
\norm{\vw^{T,\tau+1} - \vwo}_2 \leq \frac{2B\sqrt{2.0301}}{0.99c\cdot GM}
\]
or else if the above is not true, then we must instead have
\[
\norm{\vw^{T,\tau+1} - \vwo}_2 \leq 0.99\cdot\norm{\vw - \vwo}_2
\]
Note that since we have $\alpha \leq \frac c{2.88\eta + c}$, we get $\frac{2B\sqrt{2.0301}}{0.99c\cdot GM} \leq \frac1{\eta M}$. Thus, it is assured that after $t_0 = \bigO{\log\eta} = \bigO1$ iterations, iterates $\vw^{T,\tau}$ of \stir will satisfy $\norm{\vw^{T,\tau} - \vwo}_2 \leq \frac1{\eta M}$ for all $\tau \geq t_0$
\end{proof}

\begin{lemma}
\label{lem:induc-girls}
Suppose we have $n$ data points with the covariates $\vx_i$ sampled from a sub-Gaussian distribution $\cD$ and an $\alpha$ fraction of the data points are corrupted. Suppose we initialize a stage $T$ within an execution of \girls with truncation level $M$, increment parameter $\eta$, and a model $\vw^T =: \vw^{T,1}$ such that $\alpha \leq \frac c{2.88\eta + c}$ and $\norm{\vw - \vwo}_2 \leq \frac1M$, then with probability at least $1 - \exp\br{-\Om{n - d\log(d + n)}}$, there exists an upper bound of $t_0 = \bigO1$ iterations, such that we are assured that $\norm{\vw^{T,\tau} - \vwo}_2 \leq \frac1{\eta M}$ for all $\tau \geq t_0$.
\end{lemma}
\begin{proof}
As observed before, all models $\vw^{T,\tau}, \tau \geq 1$ in this stage at least satisfy $\norm{\vw^{T,\tau} - \vwo}_2 \leq \frac1M$ since the initial model $\vw^{T,1}$ satisfies this by assumption and we will see below that \girls offers monotonic convergence. Thus, Lemma~\ref{lem:wsc-wss} applies uniformly to all these models and thus, with probability at least $1 - \exp\br{-\Om{n - d\log(d + n)}}$, for all $\tau \geq 1$, the function $\wp_\frac1M(\cdot,\vw^{T,\tau})$ (refer to \S\ref{sec:huber} for notation) is $\gamma$-strongly convex for $\gamma \geq 0.99c\cdot GM$. 

Similarly, Lemma~\ref{lem:ssc-sss} tells us that, again with probability at least $1 - \exp\br{-\Om{n - d\log(d + n)}}$, for all $\tau \geq 1$,the function $\wp_\frac1M(\cdot,\vw^{T,\tau})$ is $\delta$-strongly smooth for $\delta \leq 1.01Mn$. From now on, we will be using the shorthand $\wp(\cdot) := \wp_\frac1M(\cdot,\vw^{T,\tau})$ to avoid notational clutter.

If we denote $\vg^t:= \nabla\wp(\vw^{T,\tau}) = \wp_\frac1M(\vw^{T,\tau},\vw^{T,\tau})$, then it is clear that \girls will choose as the next model as $\vw^{T,\tau+1} := \vw^{T,\tau} - \frac C{Mn}\cdot\vg^t$. For sake of notational simplicity, we will abbreviate $\vw := \vw^{T,\tau}, \vwn := \vw^{T,\tau+1}, \vg := \vg^t$. Then, applying strong smoothness tells us that
\begin{align*}
\wp(\vwn) - \wp(\vw) &\leq \ip\vg{\vwn-\vw} + \frac\delta2\norm{\vwn-\vw}_2^2\\
&= \ip\vg{\vwn-\vwo} + \ip\vg{\vwo-\vw} + \frac\delta2\norm{\vwn-\vw}_2^2\\
&= \frac{Mn}C\cdot\ip{\vw - \vwn}{\vwn - \vwo} + \ip\vg{\vwo-\vw} + \frac\delta2\norm{\vwn-\vw}_2^2\\
&= \frac{Mn}{2C}\br{\norm{\vw-\vwo}_2^2 - \norm{\vwn - \vwo}_2^2} + \ip\vg{\vwo-\vw} + \br{\frac\delta2 - \frac{Mn}{2C}}\norm{\vwn-\vw}_2^2\\
&\leq \frac{Mn}{2C}\br{\norm{\vw-\vwo}_2^2 - \norm{\vwn - \vwo}_2^2} + \ip\vg{\vwo-\vw},
\end{align*}
where the fifth step holds for any $C \leq \frac{Mn}\delta \leq 0.99$. Strong smoothness on the other hand tells us that
\[
\ip\vg{\vwo-\vw} \leq \wp(\vwo) - \wp(\vw) - \frac\gamma2\norm{\vw - \vwo}_2^2
\]
Combining the above two results gives us
\[
\wp(\vwn) - \wp(\vwo) \leq \frac{Mn}{2C}\br{\norm{\vw-\vwo}_2^2 - \norm{\vwn - \vwo}_2^2} - \frac\gamma2\norm{\vw - \vwo}_2^2
\]

Now, we can either have $\wp(\vwn) - \wp(\vwo) \geq 0$ in which case we get $\norm{\vwn - \vwo}_2 \leq \sqrt{1- \frac{C\gamma}{Mn}}\norm{\vw - \vwo}_2 \leq \sqrt{1- \frac{0.99cCG}{n}}\norm{\vw - \vwo}_2$ or else $\wp(\vwn) - \wp(\vwo) < 0$ in which case applying strong convexity once again yields
\[
\frac\gamma2\norm{\vwn-\vwo}_2^2 \leq \wp(\vwn) - \wp(\vwo) + \ip{\nabla\wp(\vwo)}{\vwo - \vwn} \leq \ip{\nabla\wp(\vwo)}{\vwo - \vwn}
\]
Now notice that $\nabla\wp(\vwo) = XS\vb$ and Lemmata~\ref{lem:bad} and \ref{lem:ssc-sss} tell us that $\norm{XS\vb}_2 \leq 2B\sqrt{5.05}$ which give us $\norm{\vwn-\vwo}_2 \leq \frac{2B\sqrt{2.0301}}\gamma \leq \frac{2B\sqrt{2.0301}}{0.99cGM} < \frac1{\eta M}$ whenever $\frac BG \leq \frac{0.99c}{2\eta\sqrt{2.0301}}$. This completes the proof of the result upon making similar arguments as those made in the proof of Lemma~\ref{lem:induc-girls}.
\end{proof}

\subsection{Bounding the Weights on Bad Points}
The following lemma establishes that neither \stir nor \girls put too much weight on bad points.
\begin{lemma}
\label{lem:bad}
Suppose during the execution of \stir or \girls, we encounter a model $\vw$ while the truncation parameter is $M$. Denote $\norm{\vw-\vw^\ast}_2 = \epsilon$ and let $S = \diag(\vs)$ be the diagonal matrix of $M$-truncated weights assigned due to residuals induced by $\vw$. Then, with probability at least $1 - \exp(-\Om{n - d})$, we must have
\[
\norm{S\vb}_2^2 \leq 4B(1 + 1.01M^2\epsilon^2),
\]
where we recall that $\vb$ denotes the vector of corruptions.
\end{lemma}
\begin{proof}
Let $\vDelta:= \vw - \vw^\ast$ and let $b_i$ denote the corruption on the data point $\vx_i$. The proof proceeds via a simple case analysis
\begin{description}
	\item[Case 1: $\abs{b_i} \leq 2\abs{\vDelta\cdot\vx_i}$] In this case we simply bound $(s_ib_i)^2 \leq M^2b_i^2 \leq 4M^2(\vDelta\cdot\vx_i)^2$.
	\item[Case 2: $\abs{b_i} > 2\abs{\vDelta\cdot\vx_i}$] In this case we have $\abs{r_i} = \abs{\vDelta\cdot\vx_i - b_i} \geq \abs{b_i} - \abs{\vDelta\cdot\vx_i} \geq \frac{\abs{b_i}}2$ and thus we must have $s_i \leq \frac2{\abs{b_i}}$ (due to possible truncation) and thus $(s_ib_i)^2 \leq 4$.
\end{description}
Thus, we get
\[
\norm{S\vb}_2^2 = \sum_{i\in B}(s_ib_i)^2 \leq 4\cdot\sum_{i\in B}\max\bc{1,M^2(\vDelta\cdot\vx_i)^2} \leq 4(B + M^2\epsilon^2\lambda_{\max}(X_BX_B^\top)) \leq 4(B + 1.01M^2\epsilon^2B),
\]
where the last step follows due to Lemma~\ref{lem:ssc-sss} which holds with probability at least $1 - \exp(-\Om{n - d})$ and finishes the proof.
\end{proof}

\subsection{Convergence with respect to Huber and Absolute Loss}
\label{app:huber-abs}
A relatively straightforward application of Theorem~\ref{thm:main} alongwith some Lipschitzness properties allows us to show that \stir and \girls also ensure convergence to the optimal objective value with respect to the Huber and absolute loss functions. These are widely used in robust regression applications.

\begin{theorem}
Under the same preconditions as those in Theorem~\ref{thm:main}, we are assured with probability at least $1- \exp(-\softOm n)$, that after $K = \bigO{\log\frac1{M_1\epsilon}}$ stages, both \stir and \girls must produce a model $\vw^K$ so that
\begin{enumerate}
	\item $\ell_\epsilon(\vw^K) \leq \ell_\epsilon(\vw^\ast) + \sqrt{1.01}\epsilon$
	\item $\frac1n\norm{X^\top\vw^K - \vy}_1 \leq \frac1n\norm{X^\top\vw^\ast - \vy}_1 + \frac{3\sqrt{1.01}}2\epsilon$.
\end{enumerate}
\end{theorem}
\begin{proof}
The first part follows directly from Lemma~\ref{lem:feps-lip} and Theorem~\ref{thm:main}. The second part follows due to the fact that $\abs x \leq f_\epsilon(x) \leq \abs x + \frac\epsilon2$ for any $\epsilon > 0$ and thus,
\[
\frac1n\norm{X^\top\vw^K - \vy}_1 \leq \ell_\epsilon(\vw^K) \leq \ell_\epsilon(\vw^\ast) + \sqrt{1.01}\epsilon \leq \frac1n\norm{X^\top\vw^\ast - \vy}_1 + \frac{3\sqrt{1.01}}2\epsilon,
\]
where the second inequality in the above chain follows from part 1 of this claim.
\end{proof}

\section{Establishing WSC/WSS - Supplementary Details}
\label{app:wsc-wss}
Recall that for any $r > 0$ and $M > 0$, $\cS_M(r)$ denotes the set of all diagonal $M$-truncated weight matrices \stir could possibly generate with respect to models residing in the radius $R$ ball centered at $\vwo$ i.e.
\[
\cS_M(r) := \bc{S = \diag(\vs), \vs_i = \min\bc{\frac1{\abs{\ip\vw{\vx_i}-y_i}},M}, \vw \in \cB_2(\vwo,r)},
\]
then we have the following result.
\begin{lemma}
\label{lem:wsc-wss}
Suppose the data covariates $X = \bs{\vx_1,\ldots,\vx_n}$ are generated from an isotropic $R$-sub-Gaussian distribution $\cD$, and $G$ denotes the set of uncorrupted points (as well as the size of that set) then there exists a constant $c$ that depends only on the distribution $\cD$ such that for any fixed value of $M > 0$,
\[
\left.
\begin{array}{r}
\P{\exists S \in \cS_M\br{\frac1M}: \lambda_{\min}(X_GS_GX_G^\top) < 0.99c\cdot GM}\\
\vspace*{-2ex}\\ 
\P{\exists S \in \cS_M\br{\frac1M}: \lambda_{\max}(X_GS_GX_G^\top) > 1.01\cdot GM}
\end{array}
\right\} \leq \exp\br{-\Om{n - d\log(d + n)}},
\]
where the constants inside $\Om{\cdot}$ are clarified in the proof. In particular, if $\cD$ is the standard Gaussian $\cN(\vzero,I_d)$, then we can take $c = 0.96$.
\end{lemma}
\begin{proof}
The bound for the largest eigenvalue follows directly due to the fact that all weights are upper bounded by $M$ and hence $X_GS_GX_G^\top \preceq M\cdot X_GX_G^\top$ and applying Lemma~\ref{lem:ssc-sss}. For the bound on the smallest eigenvalue, notice that Lemma~\ref{lem:point-conv} shows us that for any fixed $S \in \cS_M(\frac1M)$, i.e. a set of $M$-truncated weights that correspond to some fixed model $\vw \in \cB_2\br{\vwo,\frac1M}$, we have
\[
\P{\lambda_{\min}(X_GS_GX_G^\top) < 0.995 c\cdot GM} \leq 2\cdot9^d\exp\bs{-\frac{mn(0.005c)^2}{8R^4}}
\]
Recall that we let $R_X := \max_{i \in [n]}\ \norm{\vx_i}_2$ denote the maximum Euclidean length of any covariate. However, Lemma~\ref{lem:approx} shows us that if $\vw^1,\vw^2 \in \cB_2\br{\vwo,\frac1M}$ are two models such that $\norm{\vw^1 - \vw^2}_2 \leq \tau$ then, conditioned on the value of $R_X$, the following holds \emph{almost surely}.
\[
\abs{\lambda_{\min}(X_GS^1_GX_G^\top) - \lambda_{\min}(X_GS^2_GX_G^\top)} \leq 2G\tau M^2R_X^3
\]
This prompts us to initiate a uniform convergence argument by setting up a $\tau$-net over $\cB_2\br{\vwo,\frac1M}$ for $\tau = \frac c{400R_X^3M}$. Note that such a net has at most $\br{\frac{800R_X^3}c}^d$ elements by applying standard covering number bounds for the Euclidean ball \cite[Corollary 4.2.13]{Vershynin2018}. Taking a union bound over this net gives us
\begin{align*}
\P{\exists S \in \cS_M\br{\frac1M}: \lambda_{\min}(X_GS_GX_G^\top) < 0.99c\cdot GM} &\leq 2\cdot\br{\frac{7200R_X^3}c}^d\exp\bs{-\frac{mn(0.005c)^2}{8R^4}}\\
&\leq \exp\br{-\Om{n - d\log(d + n)}},
\end{align*}
where in the last step we used Lemma~\ref{lem:rx} to bound $R_X = \bigO{R\sqrt{d + n}}$ with probability at least $1 - \exp(-\Om n)$. For the specific bound on the constant $c$ for various distributions, including the Gaussian distribution, we refer the reader to Section~\ref{app:c-values}.
\end{proof}

The proof of the above result relies on several intermediate results which we prove in succession below. In the first result Lemma~\ref{lem:point-exp}, we establish expected bounds on the extremal singular values of the matrix $X_GS_GX_G^\top$ corresponding to a fixed model $\vw \in \cB_2\br{\vwo,\frac1M}$. In the next result Lemma~\ref{lem:point-conv}, we establish the same result, but this time with high probability instead of in expectation. The next result Lemma~\ref{lem:approx} establishes that extremal singular values corresponding to two models close to each other must be (deterministically) close.

\begin{lemma}[Pointwise Expectation]
\label{lem:point-exp}
With the same preconditions as in Lemma~\ref{lem:wsc-wss}, there must exist a constant $c > 0$ that depends only on $\cD$ such that for any fixed $S \in \cS_M(\frac1M)$, and fixed vector unit $\vv \in S^{d-1}$, we have
\[
c\cdot GM \leq \E{\vv^\top X_GS_GX_G^\top\vv} \leq GM.
\]
In particular, if $\cD$ is the standard Gaussian $\cN(\vzero,I_d)$, then we can take $c = 0.96$.
\end{lemma}
\begin{proof}
Let $\vx \sim \cD$ and let $y = \ip{\vwo}{\vx}$. Then if we let $\vDelta := \frac{\vw - \vwo}{\norm{\vw - \vwo}_2}$ (note that $\norm{\vw - \vwo} \leq \frac1M$), then we have $s = \min\bc{\frac1{\abs{\ip{\vw}{\vx} - y}}, M} \geq M\cdot\min\bc{\frac1{\abs{\ip\vDelta\vx}},1}$ as well as $s \leq M$. Then by linearity of expectation we have
\[
\E{\vv^\top X_GS_GX_G^\top\vv} = \E{\sum_{i\in G}\vs_i\ip{\vx_i}{\vv}^2} = G\cdot\E{s\cdot\ip\vx\vv^2} \leq GM\cdot\E{\ip\vx\vv^2} = GM,
\]
since $\cD$ is isotropic. We also get
\[
\E{\vv^\top X_GS_GX_G^\top\vv} = G\cdot\E{s\cdot\ip\vx\vv^2} \geq GM\cdot\E{\min\bc{\frac1{\abs{\ip\vDelta\vx}},1}\cdot\ip\vx\vv^2} \geq c\cdot GM,
\]
where, for any distribution $\cD$ over $\bR^d$, we define the constant $c$ as
\[
c := \inf_{\vu,\vv \in S^{d-1}}\bc{\Ee{\vx\sim\cD}{\min\bc{\frac1{\abs{\ip\vu\vx}},1}\cdot\ip\vx\vv^2}}.
\]
This concludes the proof. For the specific bound on the constant $c$ for various distributions, including the Gaussian distribution, we refer the reader to Section~\ref{app:c-values}.
\end{proof}

\begin{lemma}[Pointwise Convergence]
\label{lem:point-conv}
With the same preconditions as in Lemma~\ref{lem:wsc-wss}, for any fixed $S \in \cS_M(\frac1M)$,
\[
\left.
\begin{array}{r}
\P{\lambda_{\min}(X_GS_GX_G^\top) < 0.995 c\cdot GM}\\
\vspace*{-2ex}\\
\P{\lambda_{\max}(X_GS_GX_G^\top) > 1.005 \cdot GM}
\end{array}
\right\}
\leq 2\cdot9^d\exp\bs{-\frac{mn(0.005c)^2}{8R^4}}
\]
\end{lemma}
\begin{proof}
Note that for any square symmetric matrix $A \in \bR^{d \times d}$, we have $c - \delta \leq \lambda_{\min}(A) \leq \lambda_{\max}(A) \leq c + \delta$ for some $\delta > 0$ iff $\abs{\vv^\top A\vv - c} \leq \delta$ for all $\vv \in S^{d-1}$ which itself happens iff $\norm{A - c\cdot I}_2 \leq \delta$. Now, if $\cN_\epsilon$ denotes an $\epsilon$-net over $S^{d-1}$, then for any square symmetric matrix $B \in \bR^{d\times d}$, we have $\norm{B}_2 \leq (1-2\epsilon)^{-1}\sup_{\vv \in \cN_\epsilon}\abs{\vv^\top B\vv}$. Thus, setting $B = A - c\cdot I$ and $\epsilon = 1/4$, we have $\norm{A - c\cdot I}_2 \leq 2\sup_{\vv \in \cN_{1/4}}\abs{\vv^\top A\vv - c}$.

Let $\vx \sim \cD$ and $t = \sqrt{\min\bc{\frac1{\abs{\ip{\vw - \vwo}\vx}},M}} \leq \sqrt M$ and for any fixed $\vv \in S^{d-1}$, let $Z := t\cdot\ip\vx\vv$. Then we have
\[
\norm{Z}_{\psi_2} = \sup_{p \geq 1} p^{-1/2}\br{\E{|Z|^p}}^{1/p} \leq \sqrt M\cdot\sup_{p \geq 1} p^{-1/2}\br{\E{|\ip\vx\vv|^p}}^{1/p} = R\sqrt M,
\]
where the last step follows by observing that since $\cD$ is $R$-sub-Gaussian, $\norm{\ip{\vx_1}{\vv}}_{\Psi_2} \leq R$. Thus, $Z$ is $R\sqrt M$-sub-Gaussian. This implies $Z^2$ is $MR^2$-subexponential (see \cite[Lemma 2.7.6]{Vershynin2018}), as well as $Z^2 - \bE Z^2$ is $2MR^2$-subexponential by centering and applying the triangle inequality. Note that Lemma~\ref{lem:point-exp} implicitly establishes that $\mu := \bE Z^2 \in [cM,M]$. Let $Z_1,Z_2,\ldots,Z_G$ be independent realizations of $Z$ with respect to a fixed vector $\vv$. Then we have
\begin{align*}
\P{\abs{\vv^\top X_GS_GX_G^\top\vv - G\mu} \geq \varepsilon\cdot GM} &= \P{\abs{\sum_{i \in G}(Z^2_i - \mu)} \geq \varepsilon\cdot GM}\\
&\leq 2\exp\bs{-m\cdot\min\bc{\frac{(\varepsilon\cdot GM)^2}{4M^2R^4G},\frac{\varepsilon\cdot GM}{2MR^2}}}\\
&\leq 2\exp\bs{-\frac{mn\varepsilon^2}{8R^4}}
\end{align*}
where $m > 0$ is a universal constant and in the last step we used $G \geq n/2$ and w.l.o.g. we assumed that $\varepsilon \leq 2R^2$. Taking a union bound over all $9^d$ elements of $\cN_{1/4}$, we get
\begin{align*}
\P{\norm{X_GS_GX_G^\top - G\mu\cdot I}_2 \geq \varepsilon\cdot GM} &\leq \P{\max_{\vv\in\cN_{1/4}}\abs{\vv^\top X_GS_GX_G^\top\vv - G\mu} \geq \frac\varepsilon2\cdot GM}\\
&\leq 2\cdot9^d\exp\bs{-\frac{mn\varepsilon^2}{8R^4}}
\end{align*}
Setting $\varepsilon = 0.005c$ and noticing that $\mu \in [cM,M]$ by Lemma~\ref{lem:point-exp} finishes the proof.
\end{proof}

\begin{lemma}[Approximation Bound]
\label{lem:approx}
Consider two models $\vw^1,\vw^2 \in \bR^d$ such that $\norm{\vw^1 - \vw^2}_2 \leq \tau$ and let $\vs^1,\vs^2$ denote the $M$-truncated weight vectors they induce i.e. $s^j_i = \min\bc{M,\frac1{\abs{\ip{\vw^j}{\vx_i} - y_i}}}, j = 1,2$. Also let $S^1 = \diag(\vs^1)$ and $S^2 = \diag(\vs^2)$. Then for any $X = [\vx_1,\ldots,\vx_n] \in \bR^{d\times n}$ such that $\norm{\vx_i}_2 \leq R_X$ for all $i$,
\[
\abs{\lambda_{\min}(XS^1X^\top) - \lambda_{\min}(XS^2X^\top)} \leq 2n\tau M^2R_X^3
\]
\end{lemma}
\begin{proof}
We have the following four cases with respect to the weights $s^j_i = \min\bc{M,\frac1{\abs{\ip{\vw^j}{\vx_i} - y_i}}}, j = 1,2$ these two models generate on any data point $\vx_i \in \cB_2(R_X)$. Note that we do not assume that these data points are generated from $\cD$, just that they are bounded inside the ball $\cB_2(R_X)$. Also note that although $\abs{s^1_i - s^2_i} \leq M$ trivially holds by virtue of truncation, such a result is not sufficient for us since our later analyses would like to be able to show $\abs{s^1_i - s^2_i} \leq \frac M{1000}$ by setting $\tau$ to be really small.
\begin{description}
	\item[Case 1]: $\abs{\ip{\vw^1}{\vx_i} - y_i} \leq \frac1M$ and $\abs{\ip{\vw^2}{\vx_i} - y_i} \leq \frac1M$. Here $s^1_i = s^2_i = M$ i.e. $\abs{s^1_i - s^2_i} = 0$.
	\item[Case 2]: $\abs{\ip{\vw^1}{\vx_i} - y_i} > \frac1M$ but $\abs{\ip{\vw^2}{\vx_i} - y_i} \leq \frac1M$. In this case $s^2_i = M > s^1_i$. Thus, 
	\[
	\abs{s^1_i - s^2_i} = M - \frac1{\abs{\ip{\vw^1}{\vx_i} - y_i}} \leq M - \frac1{\abs{\ip{\vw^2}{\vx_i} - y_i} + \tau R_X} \leq M  - \frac M{1 + \tau MR_X} < 2\tau M^2R_X
	\]
	\item[Case 3]: $\abs{\ip{\vw^1}{\vx_i} - y_i} \leq \frac1M$ but $\abs{\ip{\vw^2}{\vx_i} - y_i} > \frac1M$. This is similar to Case 2 above.
	\item[Case 4]: $\abs{\ip{\vw^1}{\vx_i} - y_i} > \frac1M$ and $\abs{\ip{\vw^2}{\vx_i} - y_i} > \frac1M$. In this case we have
	\[
	\abs{\frac1{\abs{\ip{\vw^1}{\vx_i} - y_i}} - \frac1{\abs{\ip{\vw^2}{\vx_i} - y_i}}} \leq \frac{\abs{\ip{\vw^1-\vw^2}{\vx_i}}}{\abs{\ip{\vw^1}{\vx_i} - y_i}\cdot\abs{\ip{\vw^2}{\vx_i} - y_i}} \leq 2\tau M^2R_X
	\]
\end{description}
This tells us that $\norm{\vs^1 - \vs^2}_1 \leq 2n\tau M^2R_X$. Now, if we let $S^1 = \diag(\vs^1)$ and $S^2 = \diag(\vs^2)$, then for any unit vector $\vv \in S^{d-1}$, denoting $R_X := \max_{i \in [n]}\ \norm{\vx_i}_2$ we have
\[
\abs{\vv^\top XS^1X^\top\vv - \vv^\top XS^2X^\top\vv} = \abs{\sum_{i=1}^n\br{\vs^1_i-\vs^2_i}\ip{\vx_i}{\vv}^2} \leq \norm{\vs^1 - \vs^2}_1\cdot\max_{i\in[n]}\ \ip{\vx_i}{\vv}^2 \leq \norm{\vs^1 - \vs^2}_1\cdot R_X^2 \leq 2n\tau M^2R_X^3.
\]
This proves that $\norm{XS^1X^\top - XS^2X^\top}_2 \leq 2n\tau M^2R_X^3$ and concludes the proof.
\end{proof}

\subsection{Calculation of Distribution-specific Constants}
\label{app:c-values}
The WSC/WSS bounds from Lemma~\ref{lem:wsc-wss} are parametrized by a constant $c$ that lower bounds on the singular values of the matrix $X_GS_GX_G^\top$. Recall that for any covariate distribution $\cD$, the constant is defined as
\[
c := \inf_{\vu,\vv \in S^{d-1}}\bc{\Ee{\vx\sim\cD}{\min\bc{\frac1{\abs{\ip\vu\vx}},1}\cdot\ip\vx\vv^2}}.
\]
Below we present some interesting cases where this constant is lower bounded.
\begin{description}
	\item[Centered Isotropic Gaussian] For the special case of $\cD = \cN(\vzero,I_d)$, notice that by rotational symmetry, we can, without loss of generality, take $\vu = (1,0,0,\ldots,0)$ and $\vv = (v_1,v_2,0,0,\ldots,0)$ where $v_1^2 + v_2^2 = 1$. Thus, if we consider $x_1,x_2 \sim \cN(0,1)$ i.i.d. then $c \geq \inf_{(v_1,v_2) \in S^1} f(v_1,v_2)$ where
	\begin{align*}
		f(v_1,v_2) &= \Ee{x_1,x_2\sim\cN(0,1)}{\min\bc{\frac1{\abs{x_1}},1}\cdot(v_1^2x_1^2 + v_2^2x_2^2 + 2v_1v_2x_1x_2)}\\ 
		&= \Ee{x_1,x_2\sim\cN(0,1)}{\min\bc{\frac1{\abs{x_1}},1}\cdot(v_1^2x_1^2 + v_2^2x_2^2)}\\
		&= \Ee{x_1\sim\cN(0,1)}{\min\bc{\frac1{\abs{x_1}},1}\cdot(v_1^2x_1^2 + v_2^2)}\\
		&= \sqrt\frac2\pi\br{\int_0^1 (v_1^2 t^2 + v_2^2) e^{-t^2/2}dt + \int_1^\infty \br{v_1^2t + \frac{v_2^2}t} e^{-t^2/2}dt}\\
		&\geq 0.6827\cdot v_1^2 + 0.9060\cdot v_2^2
	\end{align*}
	where in the second step we used the independence of $x_1,x_2$ and $\E{x_2} = 0$, in the third step we used independence once more and $\E{x_2^2} = 1$, and in the last step we used standard bounds on the error function and the exponential integral. This gives us $c \geq \inf_{(v_1,v_2) \in S^1}\ \bc{0.6827\cdot v_1^2 + 0.9060\cdot v_2^2} \geq 0.68$.
	\item[Centered Non-isotropic Gaussian] For the case of $\cD = \cN(\vzero,\Sigma)$, we have $\vx \sim \cD \equiv \Sigma^{1/2}\cdot\cN(\vzero,I_d)$. Thus, for any fixed unit vector $\vv$, we have $\ip\vv\vx \sim \ip{\tilde\vv}\vz$ where $\tilde\vv = \Sigma^{-1/2}\vv$ and $\vz\sim\cN(\vzero,I)$. We also have $\norm{\tilde\vv}_2 \in \bs{\frac1{\sqrt\Lambda},\frac1{\sqrt\lambda}}$ where $\lambda = \lambda_{\min}(\Sigma)$ and $\Lambda = \lambda_{\max}(\Sigma)$. Note that we must insist on having $\lambda = \lambda_{\min}(\Sigma) > 0$ failing which, as the calculations show below, there is no hope of expecting $c$ to be bounded away from $0$. Now for any fixed vectors $\vu,\vv$ we first perform rotations so that we have $\tilde\vu = (u,0,0,\ldots,0)$ and $\tilde\vv = (v_1,v_2,0,0,\ldots,0)$ where we can assume w.l.o.g. that $u \geq 0$. Note that since $\bc{\norm{\tilde\vu}_2,\norm{\tilde\vv}_2} \in \bs{\frac1{\sqrt\Lambda},\frac1{\sqrt\lambda}}$, we have $(v_1,v_2) \in S^r$ and $r, u \in \bs{\frac1{\sqrt\Lambda},\frac1{\sqrt\lambda}}$. This gives us $c \geq \inf_{(v_1,v_2) \in S^r} f(v_1,v_2)$ where
	\begin{align*}
		f(v_1,v_2) &= \Ee{x_1,x_2\sim\cN(0,1)}{\min\bc{\frac1{u\cdot\abs{x_1}},1}\cdot(v_1^2x_1^2 + v_2^2x_2^2 + 2v_1v_2x_1x_2)}\\
		&= \Ee{x_1,x_2\sim\cN(0,1)}{\min\bc{\frac1{u\cdot\abs{x_1}},1}\cdot(v_1^2x_1^2 + v_2^2x_2^2)}\\
		&= \Ee{x_1\sim\cN(0,1)}{\min\bc{\frac1{u\cdot\abs{x_1}},1}\cdot(v_1^2x_1^2 + v_2^2)}\\
		&= \frac1u\sqrt\frac2\pi\br{\int_0^\frac1uu(v_1^2t^2 + v_2^2)e^{-t^2/2}dt + \int_\frac1u^\infty\br{v_1^2t + \frac{v_2^2}t}e^{-t^2/2}dt}\\
		&\geq \frac1u\sqrt\frac2\pi\br{\int_0^\frac1uu(v_1^2t^2 + v_2^2)e^{-\frac12\br{\frac1u}^2}dt + v_1^2e^{-\frac12\br{\frac1u}^2} + \frac{v_2^2}2\int_{\frac12\br{\frac1u}^2}^\infty \frac1ze^{-z}dz}\\
		&\geq \frac1u\sqrt\frac2\pi\br{e^{-\frac12\br{\frac1u}^2}\br{\frac{v_1^2}{3u^2} + v_2^2} + v_1^2e^{-\frac12\br{\frac1u}^2} + \frac{v_2^2}4e^{-\frac12\br{\frac1u}^2}\log\br{1+4u^2}}\\
		&\geq \sqrt{\frac{2\lambda}\pi}e^{-\frac\Lambda2}\br{v_1^2\br{1 + \frac\lambda3} + v_2^2\br{1 + \frac14\log\br{1+\frac4\Lambda}}}\\
		&\geq \sqrt{\frac{2\lambda}\pi}e^{-\frac\Lambda2}(v_1^2 + v_2^2)\\
		&= r^2\sqrt{\frac{2\lambda}\pi}e^{-\frac\Lambda2}\\
		&\geq \frac1\Lambda\sqrt{\frac{2\lambda}\pi}e^{-\frac\Lambda2}
	\end{align*}
	where in the second and third steps we used independence of $x_1,x_2$, $\E{x_2} = 0$ and $\E{x_2^2} = 1$ as before, and in the sixth step we used lower bounds on the exponential integral.
	\item[Non-centered Isotropic Gaussian] We discuss two techniques to handle the case of non-centered covariates.
		\begin{itemize}
			\item \textbf{Pairing Trick} This technique requires changes to the data points and relies on the fact that the difference of two i.i.d. non-centered Gaussian random variables is a centered Gaussian random variable with double the variance. Thus, given $n$ covariates $\vx_1,\ldots,\vx_n \sim \cN(\vmu,I_d)$ and corresponding responses $y_1,\ldots,y_n$, create $n/2$ data points (assume without loss of generality that $n$ is even) $\tilde\vx_i = \frac{\vx_i - \vx_{i+n/2}}{\sqrt2}$ and $\tilde y_i = \frac{y_i - y_{i+n/2}}{\sqrt2}$. Clearly $\tilde\vx_i \sim \cN(\vzero, 2\cdot I_d)$. However, this method has drawbacks since it is likely to increase the proportion of corrupted data points. If $\alpha$ fraction of the original points were corrupted, at most $2\alpha$ fraction of the new points would be corrupted.
			\item \textbf{Direct Centering} Suppose we have data from a distribution $\cD = \cN(\vmu, I_d)$. As earlier, by rotational symmetry, we can take $\vu = (1,0,0,\ldots,0), \vv = (v_1,v_2,0,0,\ldots,0)$ and $\vmu = (\mu_1,\mu_2,\mu_3,0,0,\ldots,0)$. Assume $\norm\vmu_2 = \rho$ and, without loss of generality, $\rho \geq 2$. Letting $\ip\vmu\vv =: p \leq \rho$ and $x_1,x_2,x_3 \sim \cN(0,1)$ i.i.d. gives $c \geq \inf_{(v_1,v_2) \in S^1}\ f(v_1,v_2)$ where, as before, independence of $x_1,x_2,x_3$ and the fact that $\E{x_2} = 0$ and $\E{x_2^2} = 1$, gives us
	\[
	f(v_1,v_2) = \Ee{x_1\sim\cN(0,1)}{\min\bc{\frac1{\abs{x_1 + \mu_1}},1}\cdot((p + v_1x_1)^2 + v_2^2)}
	\]
	Now, since $(v_1,v_2) \in S^1$ we get two cases (recall that we have assumed w.l.o.g. $\rho \geq 2$)
	\begin{description}
		\item[Case 1: $v_2^2 \geq \frac12$] In this case $f(v_1,v_2) \geq \frac12\Ee{x_1\sim\cN(0,1)}{\min\bc{\frac1{\abs{x_1 + \mu_1}},1}} \geq \Om{\exp^{-\rho^2/2}\log\br{1+\frac1{\rho^2}}}$.
		\item[Case 2: $v_1^2 \geq \frac12$] In this case, if $x_1 \geq 2\sqrt2\rho$, then $\abs{v_1x_1 + p} \geq \frac{v_1x_1}2$, as well as $\abs{x_1+\mu_1} \leq 2x_1$.
		\begin{align*}
			f(v_1,v_2) &\geq \Ee{x_1\sim\cN(0,1)}{\min\bc{\frac1{\abs{x_1 + \mu_1}},1}(p + v_1x_1)^2\cdot\ind{x_1 \geq 2\sqrt2\rho}}\\
			&\geq \Ee{x_1\sim\cN(0,1)}{\min\bc{\frac1{2x_1},1}\frac{x_1^2}8\cdot\ind{x_1 \geq \max 2\sqrt2\rho}} \geq \frac1{16}e^{-4\rho^2}
		\end{align*}
	\end{description}
	Since the value $\rho$ influences the final bound on $c$ very heavily, it is advisable to avoid a large $\rho$ value. One way to ensure this is to algorithmically center the covariates i.e. use $\tilde\vx_i := \vx_i - \hat\vmu$ where $\hat\vmu := \frac1n\sum_{i=1}^n\vx_i$. This would (approximately) center the covariates and ensure an effective value of $\rho \approx \bigO{\sqrt\frac d n}$
	\end{itemize}
	\item[Bounded Sub-Gaussian] Suppose our covariate distribution has bounded support i.e. $\supp(\cD) \subset \cB_2(\rho)$ for some $\rho > 0$. Assume $\rho \geq 1$ w.l.o.g. Also, using the centering trick above, assume that $\Ee{\vx\sim\cD}\vx = \vzero$. Then we have $\abs{\ip\vu\vx} \leq \rho$ which implies $\min\bc{\frac1{\abs{\ip\vu\vx}},1} \geq \frac1\rho$. Let $\Sigma$ denote the covariance of the distribution $\cD$ and let $\lambda:=\lambda_{\min}(\Sigma)$ denote its smallest eigenvalue. This gives us $c \geq \frac1\rho\Ee{\vx \sim \cD}{\ip\vx\vv^2} \geq \frac\lambda\rho$.
\end{description}

\section{Corruptions and Dense Noise - Supplementary Details}
\label{app:noisy}
In this section, we will provide details of the convergence analysis of \stir and \girls in the setting where even the ``good'' points experience sub-Gaussian noise. Thus, we will assume that our data is generated as $\vy = X^\top\vwo + \vb + \vepsilon$ where, as before $\norm\vb_0 \leq \alpha\cdot n$ and $\vepsilon \sim \cD_\varepsilon$ where $\cD_\varepsilon$ is a $\sigma$-sub-Gaussian distribution with zero mean and real support. As mentioned before, we can tolerate noise with non-zero mean as well, by using the same pairing trick we used to center the covariates in Appendix~\ref{app:c-values}. This would have a side effect of at most doubling the corruption rate $\alpha$. We will denote, as before $B := \supp(\vb)$ and $G := [n] \setminus B$. Our covariates will continue to be sampled from an $R$ sub-Gaussian distribution $\cD$ with support over $\bR^d$. We (re)state the main result of this section below.

\begin{reptheorem}{thm:dense}
Suppose we have $n$ data points with the covariates $\vx_i$ sampled from a sub-Gaussian distribution $\cD$ and an $\alpha$ fraction of the data points are corrupted with the rest subjected to sub-Gaussian noise sampled from a distribution $\cD_\varepsilon$ with sub-Gaussian norm $\sigma$. If \stir (or \girls) is initialized at an (arbitrary) point $\vw^0$, with an initial truncation that satisfies $M_1 \leq \frac1{\norm{\vw^0-\vwo}_2}$, and executed with an increment $\eta > 1$ such that we have $\alpha \leq \frac {c_\varepsilon}{5.85\eta + c_\varepsilon}$, where $c_\varepsilon > 0$ is a constant that depends only on the distributions $\cD$ and $\cD_\varepsilon$, then with probability at least $1 - \exp\br{-\Om{n - d\log(d + n) + \log\frac1{M_1\sigma}}}$, after $K = \bigO{\log\frac1{M_1\sigma}}$ stages, each of which has only $\bigO1$ iterations, we must have $\norm{\vw^K - \vwo}_2 \leq \bigO{\sigma}$.
\end{reptheorem}
\begin{proof}
The overall proof of this result follows exactly the same way as the result in Theorem~\ref{thm:main}. We will still utilize the notion of a \emph{well-initialized stage} and establish (see Lemma~\ref{lem:induc-noisy} below) a convergence guarantee for each well-initialized stage. However, Lemma~\ref{lem:induc-noisy} will itself require a few new results to be proved.

However, note that Lemma~\ref{lem:induc-stir}, a similar result for well-initialized stages in the setting without dense noise, required two results, namely Lemmata~\ref{lem:wsc-wss} and \ref{lem:bad} that established the WSC/WSS properties and bounded the weight put on bad points. Those results implicitly assumed that good points incur absolutely no modification to their response value which is no longer true here since in the setting being considered here, even good points do incur sub-Gaussian noise in their responses. Thus, we will establish below Lemmata~\ref{lem:wsc-wss-noisy} and \ref{lem:good-noisy} which will establish those results in the dense noise setting. We note that a similar convergence guarantee may be established for \girls in the dense noise setting as well.

However, note that this result only guarantees a convergence to $\norm{\vw^{K,1}- \vwo}_2 \leq \bigO\sigma$ and thus, does not ensure a consistent solution. A technical reason for this is because Lemma~\ref{lem:wsc-wss-noisy} holds true only for values of $M \leq \bigO{\frac1\sigma}$ which restricts the application of this result to offer errors much smaller than $\sigma$. It would be interesting to show, as \cite{BhatiaJKK2017} do, that \stir, or a variant, does offer consistent estimates.

For sake of notational simplicity, we will assume that $\vepsilon_B = \vzero$ by shifting any sub-Gaussian noise a bad point, say $j \in B$ does incur, into the corruption value corresponding to that point i.e. $\vb_j$. This is without loss of generality since we impose no constraints on the corruptions other than that they be sparse, in particular the corruptions need not be bounded and can thus, absorb sub-Gaussian noise values into them. 
\end{proof}

\begin{lemma}
\label{lem:induc-noisy}
Suppose we have $n$ data points with the covariates $\vx_i$ sampled from a sub-Gaussian distribution $\cD$ and an $\alpha$ fraction of the data points are corrupted with the rest experiencing noise generated i.i.d. from a distribution $\cD_\varepsilon$ with sub-Gaussian norm $\sigma$. Suppose we initialize a stage $T$ within an execution of \stir with truncation level $M \leq \frac{c_\varepsilon}{8\eta\sigma}$, increment parameter $\eta$, and a model $\vw^T =: \vw^T{T,1}$ such that $\alpha \leq \frac{c_\varepsilon}{5.85\eta + c_\varepsilon}$ and $\norm{\vw - \vwo}_2 \leq \frac1M$, then with probability at least $1 - \exp\br{-\Om{n - d\log(d + n)}}$, there exists an upper bound of $t_0 = \bigO1$ iterations, such that we are assured that $\norm{\vw^{T,\tau} - \vwo}_2 \leq \frac1{\eta M}$ for all $\tau \geq t_0$. Here $c_\varepsilon$ is the constant of the WSC property and depends only on the distributions $\cD$ and $\cD_\varepsilon$ (see Lemma~\ref{lem:wsc-wss-noisy}).
\end{lemma}
\begin{proof}
Let $\vw^{T,\tau}$ be a model encountered by \stir within this stage and let $\vr = X^\top\vw^{T,\tau} - \vy$ denote the residuals due to $\vw^{T,\tau}$ and $S = \diag(\vs)$ denote the diagonal matrix of weights where $\vs_i = \min\bc{\frac1{\abs{\vr_i}},M}$. Then \stir will choose as the next model $\vw^{T,\tau+1} = (XSX^\top)^{-1}XS\vy = \vwo + (XSX^\top)^{-1}XS(\vb+\vepsilon)$ which gives us
\[
\norm{\vw^{T,\tau+1} - \vwo}_2 \leq \frac{\norm{XS(\vb+\vepsilon)}_2}{\lambda_{\min}(XSX^\top)}
\]
Now by Lemma~\ref{lem:ssc-sss}, with probability at least $1 - \exp(-\Om{n - d})$, we have $\norm{X_B}_2 = \sqrt{\lambda_{\max}(X_BX_B^\top)} \leq \sqrt{1.01B}$. By Lemma~\ref{lem:bad}, with the same probability, we have
\[
\norm{S\vb}_2 \leq \sqrt{4B(1+1.01M^2\norm{\vw - \vwo}_2^2)} \leq 2\sqrt{2.01B},
\]
whereas by Lemma~\ref{lem:good-noisy}, as we have restricted $M \leq \frac1{8\sigma}$, we have, yet again with the same probability, 
\[
\norm{XS\vepsilon}_2 = \norm{X_GS_G\vepsilon_G} \leq 4MG\sigma\sqrt{1.01} \leq \frac{c_\varepsilon\sqrt{1.01}}{2\eta}G,
\]
where the first equality follows due to our convention that $\supp(\vepsilon) = G$ since for bad points in the set $B$, we clubbed any sub-Gaussian noise into the corruption itself, thus leaving $\vepsilon_B = \vzero$. Now, by Lemma~\ref{lem:wsc-wss-noisy}, with probability at least $1 - \exp\br{-\Om{n - d\log(d + n)}}$, we have $\lambda_{\min}(XSX^\top) \geq \lambda_{\min}(X_GS_GX_G^\top) \geq 0.99c_\varepsilon\cdot GM$. This give us
\[
\norm{\vw^{T,\tau+1} - \vwo}_2 \leq \frac{2B\sqrt{2.0301} + \frac{c_\varepsilon\sqrt{1.01}}{2\eta}G}{0.99c_\varepsilon\cdot GM} \leq \frac{2B\sqrt{2.0301}}{0.99c_\varepsilon\cdot GM} + \frac{\sqrt{1.01}}{1.98\eta\cdot M}
\]
Now, since we have $\alpha \leq \frac{c_\varepsilon}{5.85\eta + c_\varepsilon}$, we also have $\frac{2B\sqrt{2.0301}}{0.99c_\varepsilon\cdot GM} \leq \br{1-\frac{\sqrt{1.01}}{1.98}}\frac1{\eta M}$ and thus, $\frac{2B\sqrt{2.0301} + \frac{c_\varepsilon\sqrt{1.01}}{2\eta}G}{0.99c_\varepsilon\cdot GM} \leq \frac1{\eta M}$. Arguing as we did in the proof of Lemma~\ref{lem:induc-stir}, we must either have $\norm{\vwn - \vwo}_2 \leq \frac{2B\sqrt{2.0301}}{0.9801c_\varepsilon\cdot GM} + \frac{\sqrt{1.01}}{1.9602\eta\cdot M}$ and if that does not happen, we must instead have 
\[
\norm{\vw^{T,\tau+1} - \vwo}_2 \leq 0.99\cdot\norm{\vw^{T,\tau} - \vwo}_2
\]
This proves the claimed result.
\end{proof}

\subsection{Establishing WSC/WSS in Presence of Dense Noise}
We will rework a counterpart to Lemma~\ref{lem:wsc-wss} in this section.
\begin{lemma}
\label{lem:wsc-wss-noisy}
Given the problem setting above, then there exists a constant $c_\varepsilon > 0$ that depends only on the distributions $\cD,\cD_\varepsilon$ such that for any $M \in \bs{0,\frac1\sigma}$, we have
\[
\P{\exists S \in \cS_M\br{\frac1M} : \lambda_{\min}(X_GS_GX_G^\top) < 0.99c_\varepsilon\cdot GM} \leq \exp\br{-\Om{n - d\log(d + n)}}
\]
In particular, for standard Gaussian covariates and Gaussian noise with variance $\sigma^2$, we can take $c_\varepsilon \geq 0.52$.
\end{lemma}
\begin{proof}
Let $\vx \sim \cD, \epsilon \sim \cD_\varepsilon$ and let $y = \ip\vwo\vx + \epsilon$ be the response of an uncorrupted data point and $\vw \in \cB_2\br{\vwo,\frac1M}$ be any fixed model. Then if we let $\vDelta := \vw-\vwo$, the weight $s$ that the model $\vw$ would cause \stir to put on this (clean) data point must satisfy $s \geq \min\bc{\frac1{\abs{\ip\vDelta\vx - \epsilon}},M}$. This gives us, for any fixed $\vv \in S^{d-1}$,
\[
\E{\vv^\top X_GS_GX_G^\top\vv} \geq c_\varepsilon\cdot GM,
\]
where we define,
\[
c_\varepsilon := \inf_{\substack{0\leq r\leq\frac1M\\\vu,\vv\in S^{d-1}}}\bc{\Ee{\vx\sim\cD,\epsilon\sim\cD_\varepsilon}{\min\bc{\frac1{\abs{Mr\ip\vu\vx-M\epsilon}},1}\cdot\ip\vx\vv^2}}
\]
We analyze the constant $c$ for the Gaussian case at the end of the proof. For now, we proceed as in Lemma~\ref{lem:point-conv} and realize that the sub-Gaussian norm calculations continue to hold in this case since they simply upper bound the weights by $M$, and get
\[
\P{\lambda_{\min}(X_GS_GX_G^\top) < 0.995 c_\varepsilon\cdot GM} \leq 2\cdot9^d\exp\bs{-\frac{mn(0.005c_\varepsilon)^2}{8R^4}}
\]
After this we notice that the proof of Lemma~\ref{lem:approx} pays no heed to corruptions or additional noise and hence, continues to hold in this setting too. Proceeding as in the proof of Lemma~\ref{lem:wsc-wss} to set up a $\tau$-net over $\cB_2\br{\vwo,\frac1M}$ and taking a union bound over this net finishes the proof.

For the special case of $\cD = \cN(\vzero,I_d)$ and $\cD_\varepsilon = \cN(0,\sigma^2)$, by rotational symmetry, we can, without loss of generality, take $\vu = (1,0,0,\ldots,0)$ and $\vv = (v_1,v_2,0,0,\ldots,0)$ where $v_1^2 + v_2^2 = 1$. Thus, if $x_1,x_2,\epsilon \sim \cN(0,1)$ i.i.d. then $c \geq \inf_{(v_1,v_2) \in S^1, r \in \bs{0,\frac1M}} f(v_1,v_2,r)$ where
	\begin{align*}
		f(v_1,v_2,r) &= \Ee{x_1,x_2,\epsilon\sim\cN(0,1)}{\min\bc{\frac1{\abs{Mr x_1- M\sigma\epsilon}},1}\cdot(v_1^2x_1^2 + v_2^2x_2^2 + 2v_1v_2x_1x_2)}\\ 
		&= \Ee{x_1,x_2,\epsilon\sim\cN(0,1)}{\min\bc{\frac1{\abs{Mr x_1- M\sigma\epsilon}},1}\cdot(v_1^2x_1^2 + v_2^2x_2^2)}\\
		&= \Ee{x_1,\epsilon\sim\cN(0,1)}{\min\bc{\frac1{\abs{Mr x_1- M\sigma\epsilon}},1}\cdot(v_1^2x_1^2 + v_2^2)}\\
		&= v_1^2\cdot\underbrace{\Ee{x_1,\epsilon\sim\cN(0,1)}{\min\bc{\frac1{\abs{Mr x_1- M\sigma\epsilon}},1}x_1^2}}_{(A)} + v_2^2\cdot\underbrace{\Ee{z\sim\cN(0,1)}{\min\bc{\frac1{M\sqrt{r^2+\sigma^2}\abs{z}},1}}}_{(B)}
	\end{align*}
	where in the second step we used the independence of $x_1,x_2$ and $\E{x_2} = 0$, in the third step we used independence once more and $\E{x_2^2} = 1$. In the fourth step, we substituted $\sqrt{r^2+\sigma^2}z = r x_1- \sigma\epsilon$ and noticed that $r x_1- \sigma\epsilon \sim \cN(0,(r^2+\sigma^2))$ i.e. $z \sim \cN(0,1)$. To bound $(B)$ we notice $r \leq \frac1M$ and $M \leq \frac1\sigma$ and use standard bounds on Gaussian and exponential integrals to get
	\[
	(B) \geq \Ee{z\sim\cN(0,1)}{\min\bc{\frac1{\sqrt2\abs{z}},1}} \geq 0.815
	\]
	To bound $(A)$, we use the fact that pairwise orthogonal projections of a standard Gaussian vector yield independent variables. Thus, if we denote $a = Mr, b = M\sigma$ and $z = \frac{ax_1 - b\epsilon}{\sqrt{a^2+b^2}}, w = \frac{bx_1 + a\epsilon}{\sqrt{a^2+b^2}}$, then $z,w\sim\cN(0,1)$ as well as $z\perp w$. Thus, we have
	\begin{align*}
	(A) &= \Ee{z,w\sim\cN(0,1)}{\min\bc{\frac1{M\sqrt{r^2+\sigma^2}\abs z},1}\cdot\br{\frac{r^2z^2+\sigma^2w^2+2r\sigma zw}{r^2+\sigma^2}}}\\
	&\geq \Ee{z,w\sim\cN(0,1)}{\min\bc{\frac1{\sqrt2\abs z},1}\cdot\br{\frac{r^2z^2+\sigma^2w^2}{r^2+\sigma^2}}}\\
	&\geq \frac{0.52r^2}{r^2+\sigma^2} + \frac{0.815\sigma^2}{r^2+\sigma^2} = 0.52 + \frac{0.295\sigma^2}{r^2+\sigma^2}
	\end{align*}
	where in the second step we used $M \leq \frac1\sigma$ and $r \leq \frac1M$, independence of $z$ and $w$ and the fact that $\E{w} = 0, \E{w^2} = 1$ and the last step uses standard bounds on Gaussian and exponential integrals.
\end{proof}

\subsection{Bounding the Weights on Good Points}
Although Lemma~\ref{lem:bad} continues to hold in this case, since good points also incur modifications to their response values, albeit modifications that are stochastic and not adversarial, we need an analogous result for the good points in this case as well.

\begin{lemma}
\label{lem:good-noisy}
Suppose $\sigma$ is the sub-Gaussian norm of the noise distribution $\cD_\varepsilon$ and the identity of the good points $G$ is chosen independently of the covariates. Then for any $M > 0$, if $S$ is the diagonal matrix of $M$-truncated weights assigned to the data points by a model $\vw$, then with probability at least $1 - \exp(-\Om{n - d})$,
\[
\norm{X_GS_G\vepsilon_G}_2 \leq 4MG\sigma\sqrt{1.01}
\]
\end{lemma}
\begin{proof}
We have, by applying Lemma~\ref{lem:ssc-sss}, with probability at least $1 - \exp(-\Om{n - d})$,
\[
\norm{X_GS_G\vepsilon_G}_2 \leq \sqrt{\lambda_{\max}(X_GX_G^\top)}\cdot\norm{S_G\vepsilon_G} \leq \sqrt{1.01G}\cdot\norm{S}_2\norm{\vepsilon_G}_2 \leq \sqrt{1.01G}M\cdot\norm{\vepsilon_G}_2,
\]
where the last inequality follows since $S$ is a diagonal matrix and by $M$-truncation, the maximum value of any weight is $M$. Now, since our noise is $\sigma$ sub-Gaussian and unbiased, we have, for any fixed $\vu \in S^{G-1}$, $\E{\ip\vepsilon\vu} = 0$, as well as, by applying the Hoeffding's inequality,
\[
\P{\abs{\ip\vepsilon\vu} \geq t} \leq 3\exp\br{-\frac{t^2}{2\sigma^2}}
\]
Now, if $\vu^1,\vu^2 \in S^{G-1}$, such that $\norm{\vu^1-\vu^2}_2 \leq \frac12$, then we have $\abs{\ip{\vu^1-\vu^2}\vepsilon} \leq \frac12\cdot\norm\vepsilon_2$. Thus, taking a union bound over a $1/2$-net over $S^{G-1}$ gives us
\[
\P{\norm\vepsilon_2 = \max_{\vu \in S^{G-1}}\ip\vu\vepsilon \geq \frac12\cdot\norm\vepsilon_2 + t} = \P{\norm\vepsilon_2 \geq 2t} \leq 3\cdot5^G\exp\bs{-t^2/2\sigma^2}
\]
Setting $t = \sigma\sqrt{4G}$ establishes the result.
\end{proof}

\section{Robust Linear Bandits}
\label{app:wucb}

In this section, we briefly discuss the linear contextual bandit problem with corrupted arm pulls. We refer the reader to \cite{KapoorPK2018} for a more relaxed introduction to the problem as well as formal regret bounds. Indeed, the discussion here is adapted from the discussion in \cite{KapoorPK2018}.

\subsection{Problem Setting}
The stochastic linear contextual bandit framework \cite{Abbasi-YadkoriPS2011,LiCLS2010} considers a (possibly infinite) set of \emph{arms}. Arms correspond to various actions that can be performed by the algorithm. For instance, in a recommendation setting, arms may correspond to various products that are available for sale, for instance, at an e-commerce website, or in a quantitative trading setting, arms may correspond to stocks that are available for sale/purchase.

Every arm $\va$ is parametrized by a vector $\va \in \bR^d$ (we abuse notation to denote the arm and its corresponding parametrization using the same notation). Recall that the set of all arms is potentially infinite. However, not all arms may be available at every time step. For instance, an e-commerce website would not like to recommend products not currently in stock. Similarly, stocks not currently in one's possession cannot be sold.

At each time step $t$, the algorithm receives a set of $n_t$ arms (also called \emph{contexts}) $A_t = \bc{\vx^{t,1},\ldots,\vx^{t,n_t}} \subset \bR^d$ that can be played or \emph{pulled} in this round. Pulling an arm is akin o performing the action associated with that arm, for example, recommending an item or selling a stock unit. The context set $A_t$, as well as the number $n_t$ of contexts available can vary across time steps. The algorithm selects and pulls an arm $\hat\vx^t \in A_t$ as per its arm selection policy. In response, a reward $r_t$ is generated. Let $\cH^t = \bc{A_1,\hat\vx^1,r_1,\ldots,A_{t-1},\hat\vx^{t-1},r_{t-1},A_t,\hat\vx^t}$.\\

\subsection{Adversary Model}
In the stochastic linear bandit setting, as has been studied in prior work \cite{Abbasi-YadkoriPS2011,LiCLS2010} , at every time step, the reward $r_t$ is generated using a \emph{model vector} $\bto \in \bR^d$ (that is not known to the algorithm) as follows: $r_t = \ip{\bto}{\xtt} + \epsilon_t$, where $\epsilon_t$ is a \emph{noise} value that is typically assumed to be (conditionally) centered and $\sigma$-sub-Gaussian, i.e., $\E{\epsilon_t\cond\cH^t} = 0$, as well as for some $\sigma > 0$, we have $\E{\exp(\lambda\epsilon_t)\cond\cH^t} \leq \exp(\lambda^2\sigma^2/2)$ for any $\lambda > 0$.

However, recent works \cite{KapoorPK2018,LykourisML2018} have considered settings where the rewards may suffer not only sub-Gaussian noise, but also adversarial corruptions that are introduced by an \emph{adaptive adversary} that is able to view the on-goings of the online process and at any time instant $t$, \emph{after} observing the history $\cH^t$ and the ``clean'' reward value, i.e., $\ip{\bto}{\xtt} + \epsilon_t$, is able to add a corruption value $b_t$ to the reward. For notational uniformity, we will assume that for time instants where the adversary chooses not to do anything, $b_t = 0$. Thus, the final reward to the player at every time step is $r_t = \ip{\bto}{\xtt} + \epsilon_t + b_t$. This model is described in Problem Setting~\ref{algo:prob-set-lin}.

For sake of simplicity we will assume that, for some $B > 0$, the final (possibly corrupted) reward presented to the player satisfies $r_t \in [-B,B]$ almost surely. The only constraint the adversary need observe while introducing the corruptions is that at no point in the online process, should the adversary have corrupted more than an $\eta$ fraction of the observed rewards. Formally, let $G_t = \bc{\tau < t: b_\tau = 0}$ and $B_t = \bc{\tau < t: b_\tau \neq 0}$ denote the set of ``good'' and ``bad'' time instances till time $t$. We insist that $\abs{B_t} \leq \eta\cdot t$ for all $t$.

\setcounter{tmpcounter}{\value{algorithm}}
\setcounter{algorithm}{\value{modelcounter}}
\makeatletter
\renewcommand{\ALG@name}{\bfseries \small Problem Setting}
\makeatother
\begin{algorithm}[t]
	\caption{\small Adversarial Linear Bandits}
	\label{algo:prob-set-lin}
	\begin{algorithmic} 
		\FOR {$t=1,2,3..$}
		\STATE Player receives a set of contexts $A_t = \bc{\vx^{t,1},\ldots,\vx^{t,n_t}} \subset \bR^d$
		\STATE Player plays an arm, $\hat\vx^t \in A_t$
		\STATE Clean reward is generated $r^\ast_t = \ip{\bto}{\xtt} + \epsilon_t$ conditioned on $\cH^t$
		\STATE Adversary inspects $\hat\vx^t, r^\ast_t, \cH^t$ and chooses $b_t$ \COMMENT{while making sure $\abs{\tau \leq t: b_\tau \neq 0} \leq \eta\cdot (t+1)$}%
		\STATE Player receives reward, $r_t = r^\ast_t + b_t$
		\ENDFOR
	\end{algorithmic}
\end{algorithm}
\makeatletter
\renewcommand{\ALG@name}{\small Algorithm}
\makeatother
\addtocounter{modelcounter}{1}
\setcounter{algorithm}{\value{tmpcounter}}
\begin{algorithm}[t]
	\caption{\wucb: Weighted UCB for Linear Contextual Bandits}
	\label{algo:rucbl-rep}
	\begin{algorithmic}[1]
			{\small
			\REQUIRE Upper bounds $\sigma_0$ (on sub-Gaussian norm of noise distribution), $B$ (on magnitude of corruption), $\alpha_0$ (on fraction of corrupted points), initial truncation $M_1$, increment rate $\eta$%
			\FOR{$t = 1,2,\dots,T$}
				\STATE Receive set of arms $A_t$
				\STATE Play arm $\hat{\vx}^t = \underset{{\vx \in A_t,\vw \in C_{t-1}}}{\arg\max}\ \ip{\vx}{\vw}$
				\STATE Receive reward $r_t$
				\STATE $(\hvw^t, S^t) \< $ \stir$\br{\bc{\hat\vx^\tau,r_\tau}_{\tau=1}^t, M_1, \eta}$\COMMENT{Denote $S^t = \diag(s^t_1, s^t_2, \ldots, s^t_t)$}
				\STATE $V^t \< \sum_{\tau \leq t}s^t_\tau\hat{\vx}^\tau(\hat{\vx}^\tau)^\top$, $X^t \leftarrow \bs{\hat{\vx}^1, \hat{\vx}^2, \ldots, \hat{\vx}^t}$
				\STATE $\bar\vw^t \< (V^t)^{-1}X^tS^t\vy$
				\STATE $C_t \< \{\vw: \norm{\vw - \bar\vw^t}_{V^t} \leq \sigma_0\sqrt{d\log T} + \alpha_0 BT\}$
			\ENDFOR
			}
	\end{algorithmic}
\end{algorithm}

\subsection{Notion of Regret}
The goal of the algorithm is to maximize the cumulative reward it receives over the time steps $\sum_{t=1}^Tr_t$. However, a more popular technique of casting this objective is in the form of \emph{cumulative pseudo regret}. At time $t$, let $\xto = \arg\max_{\vx \in A_t}\ip{\bto}{\vx}$ be the arm among those available that yields the highest expected (uncorrupted) reward. The cumulative pseudo regret of a policy $\pi$ is defined as follows
\[
\bar R_T(\pi) = \sum_{t=1}^T\ip{\bto}{\xto} - \E{r_t}.
\]
Note that the best arm here may change across time-steps.

\subsection{\wucb: An Algorithm for Robust Linear Bandits}
We use the notation $\norm{\vx}_M = \sqrt{\vx^\top M\vx}$ for a vector $\vx \in \bR^d$ and a matrix $M \in \bR^{d \times d}$. We reproduce, for convenience, the \wucb algorithm in Algorithm~\ref{algo:rucbl-rep}. \wucb builds upon the OFUL principle \cite{Abbasi-YadkoriPS2011} for linear contextual bandits. At every step, \wucb uses rewards obtained from previous arm pulls to obtain an estimate $\btt$ of the true model vector $\bto$. 

Whereas classical algorithms utilize ordinary least squares to solve this problem, \wucb utilizes \stir (actually \girls for sake of speed) to obtain this estimate. This lends resilience to the algorithm against the (possibly several) past arm pulls whose rewards got corrupted by the adversary. The previous work of \cite{KapoorPK2018} used the \torrent algorithm for the same purpose.

The next step in executing the OFUL principle is the construction of a \emph{confidence set}. It is common to use an ellipsoidal confidence set with the ellipsoid induced by the covariance matrix of the arm vectors pulled so far. The work of \cite{KapoorPK2018} modifies this to only consider arms considered as clean by the \torrent algorithm while constructing the confidence ellipsoid.

Since \stir, instead of selecting a specific subset of arms like \torrent, instead would assign weights to all previously pulled arms, with a small weight indicating a high likelihood of the arm pull being a corrupted one and a large weight indicating a high likelihood of the arm pull being a clean one. Thus, \stir utilizes these weights to construct a \emph{weighted covariance matrix} which is then used to define the confidence ellipsoid and carry out the arm selection step.

\end{document}